\documentclass[journal,11pt,onecolumn,doublespacing]{IEEEtran}

\usepackage{url}
\usepackage{graphicx}
\usepackage{amsfonts}
\usepackage{amssymb}
\usepackage{amsmath}
\usepackage{amsthm}
\usepackage{color}

\usepackage{caption}
\usepackage[mathscr]{eucal}
\usepackage{multirow}
\usepackage{algorithm}
\usepackage{algorithmic}
\usepackage{epstopdf}
\usepackage{float}
\usepackage{subfigure} 
\usepackage{color}
\usepackage{enumitem}

\def\etal{{\em et al.\ }}

\newcommand{\R}{\mathbb{R}}

\newcommand{\diag}{\text{Diag}}

\newcommand{\iFFT}{\mathcal{F}^{-1}} 
\newcommand{\FFT}{\mathcal{F}} 

\newcommand{\NM}[2]{\| #1 \|_{#2}}
\newcommand{\NML}[2]{\left\| #1 \right\|_{#2}}

\newcommand{\sumDZ}{\sum_{k=1}^{K}d_k \! * \! z_{ik}}
\newcommand{\sumDZfre}{\sum_{k=1}^{K}\hat{d}_k\odot \hat{z}_{ik}}

\newcommand{\Px}[2]{\text{prox}_{#1}( #2 ) }

\newtheorem{theorem}{Theorem}
\newtheorem{lemma}[theorem]{Lemma}
\newtheorem{proposition}[theorem]{Proposition}

\begin{document}
\title{General Convolutional Sparse Coding with Unknown Noise}

\author
{
	Yaqing~Wang,~\IEEEmembership{Member~IEEE,}	
	James~T.~Kwok,~\IEEEmembership{Fellow~IEEE,}
	and~Lionel~M.~Ni,~\IEEEmembership{Fellow~IEEE}
	\thanks{Y. Wang, J. T. Kwok and L.M. Ni are with the Department of Computer Science and Engineering, Hong Kong University of Science and Technology University, Hong Kong.}
	\thanks{\copyright 2019 IEEE.  Personal use of this material is permitted.  Permission from IEEE must be obtained for all other uses, in any current or future media, including reprinting/republishing this material for advertising or promotional purposes, creating new collective works, for resale or redistribution to servers or lists, or reuse of any copyrighted component of this work in other works.}
}

\maketitle

\begin{abstract}
Convolutional sparse coding (CSC) can learn representative shift-invariant patterns from multiple kinds of data.
However, existing CSC methods can only model noises from Gaussian distribution, which is restrictive and unrealistic.
In this paper, we propose a general CSC model capable of dealing with complicated unknown noise. 
The noise is now modeled by Gaussian mixture model, which can approximate any continuous probability density function.  
We use the expectation-maximization algorithm to solve the problem and design an efficient method for the weighted CSC problem in maximization step. 
The crux is to speed up the convolution in the frequency domain while keeping the other computation involving weight matrix in the spatial domain.
Besides, we
simultaneously update the dictionary and codes by nonconvex accelerated proximal gradient algorithm without bringing in extra alternating loops.
The resultant method obtains comparable time and space complexity compared with existing CSC methods. 
Extensive experiments on synthetic and real noisy biomedical data sets validate that our method can model noise effectively and obtain high-quality filters and representation.
\end{abstract}

\begin{IEEEkeywords}
	Convolutional sparse coding, Noise modeling, Gaussian mixture model
\end{IEEEkeywords}

\section{Introduction}
\label{sec:intro}

\IEEEPARstart{G}{iven} a set of samples, 
sparse coding tries to 
learn an over-complete dictionary 
and then 
represent each sample as a sparse 
combination (code) of the dictionary atoms.
It has been used in various signal processing \cite{aharon2006rm,lee2007efficient} and computer vision applications \cite{mairal2009non,yang2009linear}.
Albeit its popularity, sparse coding cannot capture shifted local patterns
in the data. 
Hence, 
pre-processing (such as the extraction of sample patches) and post-processing (such as aggregating patch representation back into sample representation)
are needed, otherwise redundant representations 
will be learned. 

Convolutional sparse coding (CSC) \cite{zeiler2010deconvolutional} is a recent method which improves sparse coding by learning a shift-invariant
dictionary. This is done by replacing the 
multiplication between codes and dictionary by convolution operation, which can capture local patterns of shifted locations in the data.
Therefore, no pre-processing or post-processing are needed, and the sample can be optimized as a whole and represented as the sum of a set of filters from the dictionary
convolved with the corresponding codes. 
It has been successfully used on various data types,  including trajectories \cite{zhu2015convolutional}, images \cite{heide2015fast}, audios \cite{cogliati2016context}, videos \cite{choudhury2017consensus}, multi-spectral and light field images \cite{wang2018scsc} and biomedical data \cite{andilla2014sparse,sironi2015learning,chang2017unsupervised,jas2017learning}. 
It also succeeds on a variety of applications accompanying the data, such as 
recovering non-rigid structure from motion \cite{zhu2015convolutional},
image super resolution \cite{gu2015convolutional}, image denoising and
inpainting \cite{wang2018scsc}, 
music transcription \cite{cogliati2016context}, 
video deblurring \cite{choudhury2017consensus}, neuronal assembly detection \cite{peter2017sparse} and so on.

All above CSC works 
use square loss,
thus assume 
that the noise in the data is from Gaussian distribution.
However,
this can be restrictive and does not suit many real-world problems. 
For example, although
CSC is popularly used for biomedical data sets \cite{andilla2014sparse,chang2017unsupervised,jas2017learning,peter2017sparse} where shifting patterns abound due to cell division, it cannot handle the various complicated noises in the data.  
In fact, biomedical data sets
usually contain artifacts during recording,
e.g.,
biomedical heterogeneities,  large variations in luminance
and contrast,
and disturbance due to other small living animals
\cite{andilla2014sparse,hitziger2017adaptive}. 
Moreover, 
as the target biomedical structures
are often tiny and delicate, the existing of noises 
will heavily interfere the quality of
the learned filters and representation \cite{jas2017learning}.

Lots of algorithms have been proposed for CSC with square loss.
While the objective is not convex,
it is convex when codes of the dictionary are fixed. 
Thus, CSC is mainly solved by alternatively update 
the codes and dictionary 
by block coordinate descent (BCD) \cite{tseng2001convergence}. 
The difference of methods mainly lies in how to solve the subproblems (codes update or dictionary update) separately.
The pioneering work Deconvolutional network \cite{zeiler2010deconvolutional} uses
gradient descent for both subproblems. 
ConvCoD \cite{kavukcuoglu2010learning} uses stochastic gradient descent for dictionary update and additionally learns an encoder to output the codes. 
Recently,
other works 
\cite{heide2015fast,choudhury2017consensus,bristow2013fast,wohlberg2016efficient,sorel2016fast,papyan2017convolutional} use  
alternating direction method of multipliers (ADMM) \cite{boyd2011distributed}. 
ADMM is favored since it can decompose the subproblem into smaller ADMM subproblems which usually have closed-form solutions. 
The decomposition allows solving CSC by
separately performing faster convolution in frequency domain while enforcing the translation-variant constraints and regularizers in spatial domain.  
However, one needs another alternating loop between these ADMM subproblems so as to coordinate on the solution of the original subproblem.

Recently, 
there emerges one work which models the noise in CSC other than Gaussian.
Jas \etal 
\cite{jas2017learning} 
proposed 
the alpha-stable CSC 
($\alpha$CSC)
which 
models
the noise in 1D signals
by the symmetric alpha-stable distribution
\cite{mandelbrot1960pareto}.
This
distribution 
includes a range of heavy-tailed distributions, and is known to be more robust to noise and outliers. 
However, the 
probability density function 
of the 
alpha-stable distribution does not have an analytical form,
and its inference needs to be approximated 
by the Markov chain Monte Carlo (MCMC) \cite{gilks1995markov} procedure, which is known to be computationally expensive.
Moreover, 
as shown in Figure.~\ref{fig:probillus}, the alpha-stable distribution 
still restricts the noise to be of one particular type in advance, which is not appropriate due to unknown ground truth noise type.

\begin{figure}[ht]
	\centering
	\includegraphics[width=0.3\linewidth]{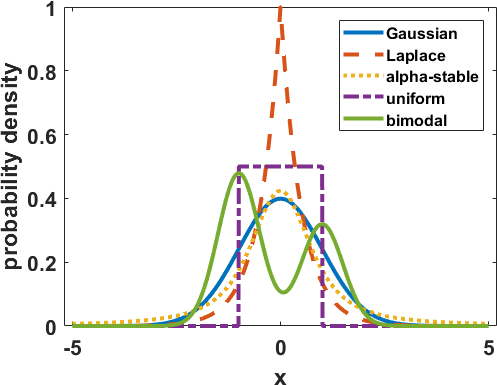}
	\vspace{-5px}
	\caption{Examples of popularly used distributions.}
	\label{fig:probillus}
\end{figure}

In this paper, we propose a general CSC model (GCSC) which enables CSC to deal with complicated unknown noise.
Specifically,
we model the noise in CSC by the Gaussian mixture model (GMM), which can approximate any continuous probability density function. 
The proposed model is then solved by the Expectation-Maximization algorithm (EM). However, the maximization step becomes a weighted variant of the CSC problem which cannot be efficiently solved by existing algorithms, e.g., BCD and ADMM, since they bring extra inner loops in M-step.
Besides, 
the weight matrix prevents us from solving the whole objective in the frequency domain to speed up the convolution.
In our proposed method, we develop a new solver to update the dictionary and codes together by a nonconvex accelerated proximal algorithm without alternating loops. 
Moreover, we manage to efficiently speed up the convolution in the frequency domain and calculate the part involving the weight matrix in the spatial domain.
The resultant algorithm achieves comparable time and space complexity compared with 
state-of-the-art CSC algorithms (for square loss).
Extensive experiments are performed on both synthetic data and real-world biomedical data
such as local field potential signals and retinal scan images. Results show that the proposed method can model the complex underlying data noise, and obtain high-quality filters and representations.

The rest of the paper is organized as follows. Section~\ref{sec:related_works}
briefly reviews CSC and proximal algorithm.
Section~\ref{sec:gcsc} describes the proposed method,
Experimental results
are presented in Section~\ref{sec:expts}, and
the last section gives some
concluding remarks.

\noindent
{\bf Notations}: 
For vector $a \in\R^m$, its $i$th element is denoted $a(i)$, its $\ell_2$-norm is $\|a\|_2 =
\sqrt{\sum_{i=1}^{m} (a(i))^2}$, its $\ell_1$-norm is
$\|a\|_1 = \sum_{i=1}^{m} |a(i)|$, and 
$\diag(a)$ reshapes $a$ to a diagonal matrix with elements $a(i)$'s.
Given another vector $b\in\R^n$, 
the convolution $a*b$ produces a vector $c\in\R^{m+n-1}$, with $c(k)=\sum_{j
	=\max(1,k+1-n)}^{\min(k,m)} a(j)b(k-j+1)$.
For matrix $A$,
$A^\top$	
denotes its
transpose, 
$A^\star$	
denotes its complex
conjugate,
$A^\dagger$ is its conjugate transpose (i.e., 
$A^\dagger=
({A^\top})^\star$).
$\odot$ denotes pointwise product.
The identity matrix is denoted $I$.
$\FFT(x)$
is 
the fast Fourier transform
that maps $x$ from the spatial domain to the frequency
domain, 
while 
$\iFFT(x)$ is the inverse operator which maps $\FFT(x)$ back to $x$.


\section{Related Work}\label{sec:related_works}


\subsection{Convolutional Sparse Coding}
\label{sec:csc}

Given $N$ 
samples
${x}_i$'s, where each ${x}_i\in \mathbb{R}^{P}$,
CSC learns a dictionary 
of $K$ filters
$d_k$'s,
each of length $M$, such that each ${x}_i$ can be well represented as
\begin{equation} \label{eq:x}
\tilde{x}_i = \sum_{k=1}^{K}d_k * z_{ik}.
\end{equation} 
Here,
$*$ is the (spatial) convolution operation, and $z_{ik}$'s  are the codes for 
$x_i$, each of length $P$.
The filters and 
codes are obtained
by solving  the following optimization problem
\begin{equation}\label{eq:csc}
\min_{\{d_k\} \in \mathcal{D}, \{z_{ik}\}} 
\sum_{i=1}^{N}\left(  \frac{1}{2}
\left\| {x}_i \!\!-\!\!  \sum_{k=1}^{K}d_k \!*\!  z_{ik} \right\|^2_2 \!\!+\!\! \sum_{k=1}^K \beta\|z_{ik}\|_1\right) ,
\end{equation}
where $\mathcal{D} = \{ D : \NM{ d_k }{2} \le 1, k = 1,\dots, K \}$ ensures
that the filters are normalized, and  the $\ell_1$ regularizer encourages the
codes to be sparse.

To solve \eqref{eq:csc}, 
block coordinate descent (BCD) \cite{tseng2001convergence} 
is typically used
\cite{zeiler2010deconvolutional,heide2015fast,kavukcuoglu2010learning,bristow2013fast,wohlberg2016efficient,sorel2016fast}.
The codes and dictionary 
are updated
in an alternating manner as follows.

\subsubsection{Code Update}
Given 
$d_k$'s, the corresponding
$z_{ik}$'s are obtained as
\begin{eqnarray}
\min_{\{z_{ik}\}}
\frac{1}{2} 
\NML{{x_i}
	-
	\sum_{k=1}^{K}{d_k}
	*
	{z_{ik}}}{2}^2
+\beta \sum_{k=1}^{K}\NM{z_{ik}}{1}.
\label{eq:csc_code_spa} 
\end{eqnarray}

Convolution can be performed much faster in the frequency domain via the convolution theorem\footnote{$\FFT(d_{k} *z_{ik} )=\FFT(d_k)\odot\FFT(z_{ik} )$, 
	where $d_k$ is first zero-padded to $P$-dimensional.} \cite{mallat1999wavelet}. Combining this with the use of  
Parseval's theorem\footnote{
For $a\in\R^P$, $\NM{a}{2}^2=\frac{1}{P}\NM{\FFT(a)}{2}^2$.
} \cite{mallat1999wavelet} and the linearity of FFT, 
problem \eqref{eq:csc_code_spa} 
is reformulated in \cite{bristow2013fast,heide2015fast,wohlberg2016efficient} as:
\begin{eqnarray}
\min_{\{z_{ik}\}}
\frac{1}{2P} 
\NML{\FFT (x_i)
	-
	\sum_{k=1}^{K}\FFT (d_k)
	\odot
	\FFT (z_{ik})}{2}^2
+\beta \sum_{k=1}^{K}\NM{z_{ik}}{1}.
\label{eq:csc_code_fre} 
\end{eqnarray} 

\subsubsection{Dictionary Update} 
Given $z_{ik}$'s, 
$d_k$'s is 
updated by solving
\begin{eqnarray}
\min_{\{d_k\}\in\mathcal{D}}
&&
\frac{1}{2} \sum_{i=1}^{N} 
\NML{{x}_i
	-
	\sum_{k=1}^{K} {d_{k}}
	*
	 {z_{ik}}}{2}^2.
\notag
\end{eqnarray}
Similar to the code update,  it is more efficient to perform
convolution in the frequency domain, as:
\begin{eqnarray}
\min_{\{d_k\}}
&&
\frac{1}{2P} \sum_{i=1}^{N} 
\NML{\FFT (x_i)
	-
	\sum_{k=1}^{K}\FFT (d_{k})
	\odot
	\FFT (z_{ik})}{2}^2
\label{eq:csc_dic_fre}
\\
\text{s.t.}&
& \|C{\iFFT(\FFT (C^\top d_k))}\|_2^2 \le 1, \forall k,
\nonumber
\end{eqnarray}
where 
$C\in \R^{M\times P}$ is a matrix with $C(i,i)=1$ and $C(i,j)=0$ for $i\neq j$ which is used to 
crop the extra dimension to recover the original spatial support, and
$C^\top$ can pad $d_k$ to be $P$-dimensional.
The constraint 
scales all filters to unit norm.

The alternating direction method of multipliers (ADMM) 
\cite{boyd2011distributed} has been commonly used 
for the code update  and dictionary update 
subproblems (\ref{eq:csc_code_fre}) 
and (\ref{eq:csc_dic_fre}))
\cite{heide2015fast,bristow2013fast,wohlberg2016efficient,sorel2016fast}.
Each ADMM subproblem has a closed-form solution that is easy to compute.
Besides, 
with the introduction of auxiliary variables,
ADMM separates computations in 
the frequency domain
(involving convolutions) and
spatial domain (involving
the $\ell_1$ regularizer and unit norm constraint).


\subsection{Proximal Algorithm}
\label{sec:prox}

The proximal algorithm \cite{parikh2014proximal} is used to solve 
composite optimization problems
of the form
\begin{equation*} \label{eq:prox} 
\min_x
F(x) \equiv 
f(x) + r(x),
\end{equation*}
where $f$ is smooth, $r$ is nonsmooth, and both are convex.
To make the proximal algorithm efficient, its underlying proximal step
(with
stepsize
$\eta$)
\begin{equation*}\label{eq:prox_step}
\Px{ \eta r }{ z } = \arg\min_{x} \frac{1}{2}\NM{x - z}{2}^2 +\eta r(x)
\end{equation*}
has to be inexpensive. 

Recently,
the proximal algorithm has been 
extended to nonconvex problems where both $f$ and $r$ can be nonconvex. 
A state-of-the-art is the nonconvex inexact accelerated proximal gradient (niAPG) algorithm \cite{yao2017efficient}, 
shown in Algorithm~\ref{alg:niAPG}.
It can efficiently converge to a critical point of the objective.

\begin{algorithm}[ht]
	\caption{Nonconvex inexact accelerated proximal gradient (niAPG) algorithm \cite{yao2017efficient}.}
	\begin{algorithmic}[1]
		\REQUIRE $x^{0} = x^{1} = 0$ and $\eta$;
		\FOR{$\tau = 1, \dots, J $}
		\STATE $v^{\tau} = x^{\tau} + \frac{\tau - 1}{\tau + 2} (x^{\tau} - x^{\tau - 1})$;		
		\STATE $\Delta^\tau = \max_{t = \max(1, \tau - 5), \dots, \tau} F(x^{t})$;
		
		\IF{$F(v^{\tau}) \ge \Delta^\tau$}
		\STATE $v^{\tau} = x^{\tau}$;
		\ENDIF
		\STATE $x^{\tau+1} = \Px{ \eta r  }{v^{\tau}- \eta \nabla f(v^{\tau}) }$;
		\ENDFOR 
		\RETURN $x^{J+1}$.
	\end{algorithmic}
	\label{alg:niAPG}
\end{algorithm}


\section{Proposed Method}
\label{sec:gcsc}

The square loss in 
(\ref{eq:csc}) implicitly assumes that the noise is
normally distributed.
In this section, we relax this assumption, and assume the noise to be generated
from
a Gaussian mixture model (GMM).
It is well-known that 
a GMM
can approximate any continuous probability density function 
\cite{maz1996approximate}. 

As in other applications of GMM, we will
use the EM algorithm for inference. However, as will be shown,
the M-step involves a difficult weighted CSC problem.
In Section~\ref{sec:wcsc}, 
we design an efficient solver based on nonconvex accelerated proximal algorithm,
with comparable time and space complexity as
state-of-the-art CSC algorithms (for 
the square loss).


\subsection{GMM Noise}
\label{sec:gcsc_formulation}

We assume that the 
noise $\epsilon_{i}$ 
associated with ${x}_i$
follows the GMM distribution:
\begin{equation*}
p(\epsilon_{i}) = \sum_{g=1}^G p(\epsilon_{i}|\phi_{i}=g) p(\phi_{i}=g),
\end{equation*}
where $G$ is the number of Gaussian components, 
$\phi_{i}$
is the latent variable 
denoting which Gaussian component 
$\epsilon_{i}$ belongs to, and
$\pi_g$'s are mixing coefficients
with $\sum_{g=1}^G\pi_g =1$.
Variable
$\phi_{i}$ follows the multinomial distribution $\text{Multinomial}(\{\pi_g\})$,
and the conditional distribution of
$\epsilon_{i}$ given $\phi_{i}=g$ follows
the normal distribution 
$\mathcal{N}(\mu_g,\Sigma_{g})$ 
with 
mean 
$\mu_g$ 
and diagonal 
covariance matrix
$\Sigma_g=\diag(\sigma^2_{g}(1), \dots, \sigma^2_{g}(P))$.
The $\ell_1$ regularizer in  (\ref{eq:csc})
corresponds to 
the prior Laplace distribution 
($\text{Laplace}(0,\frac{1}{\beta})$) on each $p$th element of $z_{ik}$:
$
p(z_{ik}(p)) = \frac{\beta}{2}\exp(-\beta|z_{ik}(p)|)
$. 


\subsection{Using Expectation Maximization (EM) Algorithm}

Let $\Theta$ denote the collection of all parameters
$\pi_g$'s, $\mu_g$'s, $\Sigma_{g}$'s, $d_k$'s and $z_{ik}$'s.
The log posterior probability 
for 
$\Theta$
is:
\begin{eqnarray}
\log \mathcal{P} 
= \sum_{i=1}^N \left(\log \sum_{g=1}^G p({x}_{i}|\phi_{i}=g)
+\log \sum_{g=1}^G
p(\phi_{i}=g)+\sum_{k=1}^K \sum_{p=1}^P\log p(z_{ik}(p))\right).
\label{eq:gcsc_llh}
\end{eqnarray}
This can be maximized by the 
Expectation Maximization (EM) algorithm 
\cite{dempster1977maximum}.

The E-step computes
$p(\phi_{i}\!=\!g|{x}_{i})$,
the posterior probability that $\phi_{i}$ belongs to the $g$th Gaussian
given ${x}_i$.
Using Bayes rule, we have
\begin{eqnarray}
p(\phi_{i}\!=\!g|{x}_{i})
&=&\!\!
\frac{\pi_g p({x}_{i}|\phi_{i}\!=\!g)}
{\sum_{g=1}^G \pi_g p({x}_{i}|\phi_{i}=g)},
\label{eq:estep}
\end{eqnarray}
where 
${x}_{i}|(\phi_{i}=g) \sim
\mathcal{N}(
\tilde{x}_i+\mu_g,
\Sigma_{g})$, and 
$\tilde{x}_i = \sum_{k=1}^{K}d_k * z_{ik}$
in \eqref{eq:x}.

The M-step obtains $\Theta$ by maximizing  
the following upper bound of 
$\log \mathcal{P}$ in \eqref{eq:gcsc_llh} as
\begin{align} \nonumber
\lefteqn{\arg\max_{\Theta}\sum_{i=1}^N  \left( \sum_{g=1}^G \gamma_{gi}\log \frac{
		p({x}_{i},\phi_{i})}{\gamma_{gi}}\!\!+\!\!\sum_{k=1}^K \sum_{p=1}^P\log
		p(z_{ik}(p)) \right)}\\
\nonumber \\
\!\!& =\!\!\arg\max_{\Theta} \!\! \sum_{i=1}^N \left(\sum_{g=1}^G \gamma_{gi}\log
p({x}_{i}|\phi_{i}=g) \!\!+\!\!\beta\sum_{k=1}^K \NM{z_{ik}}{1}\right) \nonumber \\\label{eq:ub}
\!\!& =\!\! \arg\max_{\Theta}  \!\!\sum_{i=1}^{N}
\left(\!\!
\beta\sum_{k=1}^K \NM{z_{ik}}{1}
\!\!+\!\!\sum_{g=1}^G \gamma_{gi} \log
\pi_g\!\!-\!\!\frac{\gamma_{gi}}{2}\log(|\Sigma_{g}|) 
\!\!-\!\!\frac{\gamma_{gi}}{2}\!\!\left({x}_{i}\!\!-\!\!\sumDZ\!\!-\!\!\mu_g\!\!\right)^\top
\!\!\!\!\Sigma_{g}^{-1}\!\!\left({x}_{i}\!\!-\!\!\sumDZ\!\!-\!\!\mu_g\right)\!\!\! \right)\!\!,\!\!\! 
\end{align} 
where
$\gamma_{gi}=p(\phi_{i}=g|{x}_{i})$. 


\subsubsection{Updating $\pi_g$'s, $\mu_g$'s, $\Sigma_{g}$'s
in $\Theta$}

Given
$d_k$'s, $z_{ik}$'s and $\gamma_{gi}$'s, let $\tilde{x}_i=\sumDZ$ as in \eqref{eq:x}, we obtain $\pi_g$'s, $\mu_g$'s and $\Sigma_{g}$'s by optimizing (\ref{eq:ub})
as:
\begin{align}
\max_{\{\pi_g,\mu_g,\Sigma_g\}} \sum_{i=1}^{N} \sum_{g=1}^G \left( \gamma_{gi} \log
\pi_g-\frac{\gamma_{gi}}{2}\log(|\Sigma_{g}|)
-\frac{\gamma_{gi}}{2}({x}_{i}-\tilde{x}_i-\mu_g)^\top\Sigma_{g}^{-1}({x}_{i}-\tilde{x}_i-\mu_g)
\right).
\notag
\end{align}
Taking the derivative of the objective
to zero,
the following closed-form solutions 
can be easily obtained:
\begin{eqnarray}
\pi_g &=&\frac{1}{N}\sum_{i=1}^N \gamma_{gi},
\label{eq:gmm_em_m_pi} \\
\mu_g  &=&\frac{\sum_{i=1}^N  \gamma_{gi}{x}_{i}}{\sum_{i=1}^N \gamma_{gi}},
\label{eq:gmm_em_m_mu}
\\
\Sigma_g &=&
\frac{\sum_{i=1}^{N}
	\gamma_{gi}
	({x}_{i}-\tilde{x}_i-\mu_g)({x}_{i}-\tilde{x}_i-\mu_g)^\top}
	{\sum_{i=1}^{N}\gamma_{gi}}.
\label{eq:gmm_em_m_sigma}
\end{eqnarray}


\subsubsection{Updating $d_k$'s and $z_{ik}$'s
in $\Theta$}

Given $\pi_g$'s, $\mu_g$'s, $\Sigma_{g}$'s and $\gamma_{gi}$'s, we obtain
$d_k$'s and $z_{ik}$'s 
from (\ref{eq:ub})
as:
\begin{align*}
\min_{\{d_k\}\in\mathcal{D}, \{z_{ik}\}} \!\!
\sum_{i=1}^{N}\!\!
\left( \!\!\beta\sum_{k=1}^K \NM{z_{ik}}{1}
\!\!-\!\!\sum_{g=1}^G\frac{\gamma_{gi}}{2}\left({x}_{i}\!\!-\!\!\sumDZ\!\!-\!\!\mu_g\right)^\top
\Sigma_{g}^{-1}\left({x}_{i}\!\!-\!\!\sumDZ\!\!-\!\!\mu_g\right)\right).
\end{align*}
This can be rewritten as 
\begin{align}
\!\!\!\!\!\!
\min_{\{ \! d_k \! \}
\in
\mathcal{D}, \{ \! z_{ik} \! \}}
\!\!\!\!\!\!\!
F(\{  d_k  \}, \! \{ z_{ik} \})
\!\equiv\!
 f(\{d_k\},\! \{z_{ik}\})
\!+\!r(\{d_k\}, \! \{z_{ik}\}),
\label{eq:wcsc_F}
\end{align}
where
\begin{eqnarray}
\!\!
f(\{ d_k  \}, \! \{  z_{ik}  \})
\! \equiv \!
\frac{1}{2} \! \sum_{i=1}^{N} \! \sum_{g=1}^{G}
\| w_{gi} \! \odot \! ({x}_i \! - \! \sumDZ \! - \! \mu_g) \|_F^2, 
\label{eq:f}
\end{eqnarray}
with
$w_{gi}(p)=\sqrt{\frac{\gamma_{gi}}{\sigma^2_{g}(P)}}$,
and
\begin{equation} \label{eq:r}
r(\{d_k\}, \{z_{ik}\}) \equiv \beta\sum_{k=1}^{K}\|z_{ik}\|_1 + I_{\mathcal{D}}(\{d_k\}). 
\end{equation} 
Here,
$I_{\mathcal{D}}(\cdot)$ is the indicator function on $\mathcal{D}$
(i.e., $I_{\mathcal{D}}(\{d_k\}) = 0$ if $\{d_k\}\in \mathcal{D}$, and $\infty$
otherwise).

Compared with the standard CSC problem in (\ref{eq:csc}),
problem (\ref{eq:wcsc_F}) can be viewed as a weighted CSC problem (with weight $w_{gi}$).
A CSC variant 
\cite{heide2015fast}
puts weights on $\sum_{k=1}^{K}d_k*z_{ik}$, while ours are on
${x}_i -\sumDZ-\mu_g$.
In \cite{jas2017learning},
the model 
also leads to a weighted CSC problem.
However, the authors there mentioned that 
it is not clear
how to solve a weighted CSC problem in the frequency domain.
Instead,
they resorted to solving it in the spatial domain, which is
less efficient
as discussed in Section~\ref{sec:csc}.


\subsection{Solving the Weighted CSC Subproblem \eqref{eq:wcsc_F}}
\label{sec:wcsc}

In this section, we solve the weighted CSC problem in (\ref{eq:wcsc_F}) 
using niAPG \cite{yao2017efficient} 
(Algorithm~\ref{alg:niAPG}). 
Note that the weights $w_{gi}$'s in $f$ (\ref{eq:f}) is the cause that 
prevents us from transforming the whole  objective (\ref{eq:wcsc_F}) to the frequency domain. 
Recall that \eqref{eq:csc_code_spa} is transformed to \eqref{eq:csc_code_fre} by first transforming everything in $\ell_2$ norm to frequency domain by Parseval's theorem, separately computing $\FFT(x_i)$ and $\FFT(\sumDZ)$ by linearity of FFT,  then replacing the convolution in $\FFT(\sumDZ)$ by pointwise product using convolution theorem. However, with $w_{gi}$'s, $\FFT(w_{gi}\odot(\sumDZ))$ cannot use convolution theorem to speed up.
Therefore, the key in designing an efficient solver is to only transform terms involving convolutions to
the frequency domain,
while leaving the weight $w_{gi}$ in spatial domain.
Hence we replace the ${x}_i \!-\!\sumDZ\!-\!\mu_g$ term in (\ref{eq:f})
by $\mathcal{F}^{-1} { (\FFT{({x}_i-\mu_g)} - \sum_{k=1}^{K}\FFT {(C^\top{d_k})}\odot \FFT {(z_{ik})} )}$.

The 
core steps in the niAPG algorithm are computing  (i) the gradient
$\nabla f(\cdot)$ w.r.t. 
$d_k$'s and $z_{ik}$'s
, and 
(ii) 
the proximal step $\text{prox}_{\eta r}(\cdot)$.  
We 
first 
introduce the following Lemmas.

\begin{lemma}\label{lemma:pointwise_mul}
Let $f(x)=a\odot x$ for
	$a,x\in\mathbb{C}^{P}$.
Then,
$\nabla_x f(x) =\diag(a^{\star})$, which
	reshapes $a^{\star}$ to a diagonal matrix with elements $a^{\star}(p)$'s.
\begin{proof}
$f(x)=Ax$, where
$A=\diag(a)$. Then,
$\nabla_x f(x)= A^\dagger=
\diag(a^{\star})$.
	\end{proof}
\end{lemma}
\begin{lemma}\label{lemma:fft}\cite{cooley1969fast}
For $x\in\R^P$, $\FFT(x)=\Phi x$ and $\iFFT(x)=\frac{1}{P} {\Phi}^\dagger x $, where $\Phi=[\frac{\omega^{jk}}{P}] \in \R^{P\times P}$, $\omega=e^{\frac{-2\pi i}{P}}$ is the $P$th root of unity, and $i=\sqrt{-1}$. 
Moreover, $\nabla_x\FFT(x) = \Phi^\dagger = P\iFFT(\cdot)$ and $\nabla_x \iFFT(x) = \frac{1}{P} {\Phi} = \frac{1}{P}\FFT(\cdot)$.
\end{lemma}

\begin{proposition}
For $f$ in 
(\ref{eq:f}),
\begin{eqnarray*}
\frac{\partial f(\{d_k\}, \{z_{ik}\}) }{\partial{d_{k}}}
&
\!=\! &
-C
\iFFT
\left(
\sum_{i=1}^N (\FFT (z_{ik}))^{\star}\!\odot\!
\FFT (u_i)
\right),\\
\frac{\partial f(\{d_k\}, \{z_{ik}\}) }{\partial{z_{ik}}}
&
\!=\! &
-\iFFT
((\FFT (d_{k}))^{\star}\odot
\FFT
(
u_i)
),
\end{eqnarray*} 
where
$u_i
=
\sum_{g=1}^G
w_{gi}
\odot
w_{gi}
\odot
\mathcal{F}^{-1} 
( \FFT({x}_i-\mu_g) 
	-
	\sum_{k=1}^{K}\FFT (C^\top{d_k})\odot \FFT ({z_{ik}})
	)$.
\end{proposition}

\begin{proof}
$f$
can be rewritten as
$f(\{d_k\}, \{z_{ik}\})
= \sum_{i=1}^{N} \sum_{g=1}^{G}\frac{1}{2}\NM{g_1}{2}^2$, 
where 
$g_{1}  =   w_{gi}\odot g_{2}$,
$g_{2}  =   \iFFT (g_{3})$, 
$g_{3}  =  \FFT({x}_i-\mu_g) - g_4$,
$g_{4}  =  \sum_{k=1}^{K}g_{6k}\odot g_{5k}$,
$g_{5k}  = \FFT({z_{ik})}$,
$g_{6k}  =  \FFT (g_{7k})$,
$g_{7k}  =  C^\top{d_k}$.

Using Lemma~\ref{lemma:pointwise_mul}, 
$\frac{\partial g_1}{\partial g_2} = \diag(w_{gi})$, 
$\frac{\partial g_{4}}{\partial g_{5k}} = \diag(g^{\star}_{6k})$,
and
$\frac{\partial g_{4}}{\partial g_{6k}} =\diag( g^{\star}_{5k})$.
Using Lemma~\ref{lemma:fft},
$\frac{\partial g_2}{\partial g_3} =\frac{1}{P} {\Phi}$,
$\frac{\partial g_{5k}}{\partial z_{ik}} =\Phi^\dagger$,
$\frac{\partial g_{6k}}{\partial g_{7k}} =\Phi^\dagger$.
Finally $\frac{\partial g_3}{\partial g_4} =-1$, $\frac{\partial g_{7k}}{\partial d_{k}} =C$ and
$\frac{\partial f(\{d_k\}, \{z_{ik}\})}{\partial g_{1}} = \sum_{i=1}^N\sum_{g=1}^G g_1 = \sum_{i=1}^N\sum_{g=1}^G w_{gi}
\odot
\mathcal{F}^{-1} 
( \FFT({x}_i-\mu_g) 
-
\sum_{k=1}^{K}\FFT (C^\top{d_k})\odot \FFT ({z_{ik}})
)$.

Combining all these, 
using chain rule for denominator layout, 
we obtain
\begin{align}\notag
\lefteqn{\frac{\partial f(\{d_k\}, \{z_{ik}\}) }{\partial{d_{k}}}}\\\notag
&
\!=\!
\frac{\partial g_{7k}}{\partial d_{k}}
\frac{\partial g_{6k}}{\partial g_{7k}}
\frac{\partial g_{4}}{\partial g_{6k}}
\frac{\partial g_{3}}{\partial g_{4}}
\frac{\partial g_{2}}{\partial g_{3}}
\frac{\partial g_{1}}{\partial g_{2}}
\frac{\partial f(\{d_k\}, \{z_{ik}\})}{\partial g_{1}}
\\\label{eq:gra_d}
&
\!=\!
-C
\iFFT
\left(
\sum_{i=1}^N (\FFT (z_{ik}))^{\star}\!\odot\!
\FFT (u_i)
\right),\!\\
\notag
\lefteqn{\frac{\partial f(\{d_k\}, \{z_{ik}\}) }{\partial{z_{ik}}}}\\\notag
&
\!=\!
\frac{\partial g_{5k}}{\partial z_{ik}}
\frac{\partial g_{4}}{\partial g_{5k}}
\frac{\partial g_{3}}{\partial g_{4}}
\frac{\partial g_{2}}{\partial g_{3}}
\frac{\partial g_{1}}{\partial g_{2}}
\frac{\partial f(\{d_k\}, \{z_{ik}\})}{\partial g_{1}}
\\\label{eq:gra_z}
&
\!=\!
-\iFFT
\left	
((\FFT (d_{k}))^{\star}\odot
\FFT
(
u_i)
\right).
\end{align} 
\end{proof}

Note that 
$r(\{d_k\}, \{z_{ik}\})$ 
in (\ref{eq:r})
is separable\footnote{
$r(x,y)$ is
separable if 
$r(x,y)
=r_1(x)+r_2(y)$.}.
This simplifies the 
associated proximal step,
as shown by the following Lemma.
\begin{lemma} \cite{parikh2014proximal}
	\label{lemma:sep_prox}
If $r(x,y)$ is separable,
$\Px{r}{v,w}=\Px{r_1}{v}+\Px{r_2}{w}$.
\end{lemma}
Using Lemma~\ref{lemma:sep_prox},
the component proximal steps 
can be easily computed in closed form as \cite{parikh2014proximal}:
$
\Px{ \eta I_{\mathcal{D}}}{ d_k }
 = d_k /\max( \|d_k \|_2, 1)$
, and
$\Px{\beta\eta \|\cdot\|_1}{z_{ik}(p)}
 =
\text{sign}(z_{ik}(p))\odot\max(|z_{ik}(p)| - \beta\eta, 0)$.
In the sequel, we 
avoid tuning
$\eta$ by using line search \cite{grippo2002nonmonotone}, which also 
speeds up convergence 
empirically.
The procedure for the 
solving the weighted CSC subproblem (\ref{eq:f})
is shown in
Algorithm~\ref{alg:wcsc}.
The whole algorithm, which will be called general CSC (GCSC), is shown in Algorithm~\ref{alg:gcsc}.

\begin{algorithm}[ht]
\caption{Solving the weighted CSC subproblem \eqref{eq:wcsc_F}.}
	\begin{algorithmic}[1]
		\REQUIRE  $d^{0}_{k} = d^{1}_{k} = 0, \forall~k$, $z^{0}_{ik}
		= z^{1}_{ik} = 0, \forall~i,k$, $\eta$;
		\FOR{$\tau = 1, \dots, J $}
		\STATE $u^{\tau}_{k} = d^{\tau}_{k} + \frac{\tau - 1}{\tau + 2} (d^{\tau}_{k} - d^{\tau - 1}_{k}), \forall k$;		
		\STATE $v^{\tau}_{ik} = z^{\tau}_{ik} + \frac{\tau - 1}{\tau + 2} (z^{\tau}_{ik} - z^{\tau - 1}_{ik}), \forall i,k$;
		\STATE $\Delta_\tau = \max_{t = \max(1, \tau - 5), \dots, \tau}  F(\{u^{t}_{k}\}, \{ v^{t}_{ik}\})$;
		
		\IF{$ F(\{u^{\tau}_{k}\}, \{ v^{\tau}_{ik}\}) \ge \Delta_\tau$}
		\STATE $u^{\tau}_{k} = d^{\tau}_{k}$, $\forall k$;
		\STATE $v^{\tau}_{ik} = z^{\tau}_{ik}$, $\forall i,k$;
		\ENDIF
		
		\STATE $d^{\tau+1}_{k} = \Px{ \eta I_{ \mathcal{D} } }{u^{\tau}_{k}- \eta  \frac{\partial f(\{d_k\}, \{z_{ik}\}) }{\partial{d_{k}}}}$;		
		\STATE $z^{\tau+1}_{ik} = \Px{\beta\eta \|\cdot\|_1}{v^{\tau}_{ik} - \eta 
		\frac{\partial f(\{d_k\}, \{z_{ik}\})}{\partial{z_{ik}}}}$;
		
		\ENDFOR 
		\RETURN $\{d^{J+1}_{k}\}, \{z^{J+1}_{ik}\}$. 
	\end{algorithmic}
	\label{alg:wcsc}
\end{algorithm}

\begin{algorithm}[htbp]
	\caption{General CSC with GMM loss (GCSC).}\label{alg:gcsc}
	\begin{algorithmic}[1]
		\REQUIRE $\{{x}_i\}$, $\{\pi_{g},\mu_g,\Sigma_g\}$, $\{d_k\}, \{z_{ik}\}$;
		\WHILE {not converged}
		\STATE
		\textbf{E-step}: compute  $p(\phi_{i}\!=\!g|{x}_{i})$ by \eqref{eq:estep};
		\STATE
		\textbf{M-step}: update $\{{\pi_{g},\mu_g,\Sigma_g}\}$ by \eqref{eq:gmm_em_m_pi}, \eqref{eq:gmm_em_m_mu} and \eqref{eq:gmm_em_m_sigma}; 
		update 
		$\{d_k\}$ and $\{z_{ik}\}$ by Algorithm~\ref{alg:wcsc};
		\ENDWHILE
		\RETURN $\{d_k\}, \{z_{ik}\}$.
	\end{algorithmic}
\end{algorithm}

\begin{table*}[htb]
	\caption{Comparing the proposed GCSC  with existing 
		CSC algorithms.
	}
	\centering
	\vspace{-5px}
	\begin{tabular}{c|c|c|c}
		\hline
		method                   & convolution operation &  noise modeling  &                                algorithm                                 \\ \hline
		DeconvNet \cite{zeiler2010deconvolutional} &        spatial        &     Gaussian     &          BCD, updating codes and dictionary by gradient descent          \\
		ConvCod \cite{kavukcuoglu2010learning} &        spatial        &     Gaussian     &          BCD, updating dictionary by gradient descent and codes by encoder         \\		
		FCSC \cite{bristow2013fast}         &       frequency       &     Gaussian     &                BCD, updating codes and dictionary by ADMM                \\
		FFCSC \cite{heide2015fast}         &       frequency       &     Gaussian     &                BCD, updating codes and dictionary by ADMM                \\
		CBPDN \cite{wohlberg2016efficient}     &       frequency       &     Gaussian     &                BCD, updating codes and dictionary by ADMM                \\
		CONSENSUS \cite{sorel2016fast}       &       frequency       &     Gaussian     &                BCD, updating codes and dictionary by ADMM                \\
		SBCSC \cite{papyan2017convolutional}    &        spatial        &     Gaussian     & BCD, updating dictionary by ADMM  and codes by LARS\cite{efron2004least} \\
		OCDL-Degraux \cite{degraux2017online}    &      spatial        &     Gaussian     &                  BCD, updating codes and dictionary by projected BCD                  \\
		OCDL-Liu \cite{liu2017online}        &       frequency       &     Gaussian     &     BCD, updating codes and dictionary by  proximal gradient descent     \\
		OCSC \cite{wang2018ocsc}          &       frequency       &     Gaussian     &                BCD, updating codes and dictionary by ADMM                \\
		CCSC \cite{choudhury2017consensus}     &       frequency       &     Gaussian     &                BCD, updating codes and dictionary by ADMM                \\
		$\alpha$CSC \cite{jas2017learning}     &        spatial        &   alpha-stable   & BCD, updating codes and dictionary by L-BFGS \cite{byrd1995limited}                                                                \\
		GCSC                    &       frequency       & Gaussian mixture &     niAPG                                                                    \\ \hline
	\end{tabular}
	\label{tab:cmp}
\end{table*}


\subsection{Complexity Analysis}
\label{sec:gcsc_complexity}

In each EM iteration,
the E-step in
\eqref{eq:estep} takes
$O(GNP)$ time. The M-step is dominated by gradient computations in \eqref{eq:gra_d} and (\ref{eq:gra_z}). These take $O(NKP\log P)$ time for the underlying FFT and inverse FFT operations, and 
$O(GNP)$ time for
the pointwise product.
Thus, 
each EM iteration 
takes a total of
$O(JGNP+ JNKP\log P)$ time, where $J$ is the number of niAPG iterations.
Empirically, $J$ is around 50.
As for space,  this is dominated by the 
$K$ $P$-dimensional codes for each of the $N$ samples, leading to a
space complexity of $O(NKP)$.

In comparison, the state-of-the-art batch CSC method 
(which uses the square loss)
\cite{heide2015fast}
takes
$O(NK^2P +NKP\log P)$ time per iteration
and  
$O(NKP)$ space.
Usually, $JG\ll K^2$. 


\subsection{Discussion with Existing CSC Works}
\label{sec:gcsc_discussion}

Table~\ref{tab:cmp}
compares GCSC with existing  
CSC algorithms.
The key differences are in noise modeling and algorithm design.
First, 
all methods
except GCSC and $\alpha$CSC
model the noise by Gaussian distribution, and  
$\alpha$CSC uses symmetric alpha-stable distribution. 
Recall that GMM can approximate any continuous distribution, the noises considered previously all are special case of GMM noise.
Second, 
all algorithms except GCSC use BCD, 
which alternatively updates codes and  dictionary, and a majority of methods then update the codes and dictionary by ADMM separately. 
As GCSC already has one alternating loop between E-step and M-step, 
using BCD and then ADMM will bring in two more alternating loops, resulting a much slower algorithm compared with existing CSC algorithms.
Therefore, 
we 
use niAPG 
to directly update codes and dictionary together. 
Empirical results in next section validate the efficiency of solving the weighted CSC problem \eqref{eq:wcsc_F} in GCSC by niAPG, rather than BCD. 


\setcounter{table}{2} 
\begin{table*}[htb]
	\caption{ Performance on the synthetic data.
		The best and comparable results (according to the pairwise t-test with 95\% confidence) are highlighted in bold.}
	\begin{center}
		\vspace{-10px}
		\begin{tabular}{cc|c|c|c}
			\hline
			&&MAE&RMSE&time (seconds)\\\hline
			\multirow{5}{*}{no noise}			
			&CSC-$\ell_2$&\textbf{0.000359$\pm$0.000027}&\textbf{0.000847$\pm$0.000109}&\textbf{419.64$\pm$37.86}\\\cline{2-5}
			&CSC-$\ell_1$&\textbf{0.000364$\pm$0.000171}&\textbf{0.000871$\pm$0.000183}&651.67$\pm$98.94\\\cline{2-5}
			&$\alpha$CSC&\textbf{0.000362$\pm$0.000109}&\textbf{0.000868$\pm$0.000122}&2217.44$\pm$345.96\\\cline{2-5}
			&GCSC &\textbf{0.000330$\pm$0.000125}&\textbf{0.000849$\pm$0.000141}&\textbf{414.34$\pm$34.48}\\\hline		
			\multirow{5}{*}{Gaussian noise}			
			&CSC-$\ell_2$&\textbf{0.00368$\pm$0.00036}&\textbf{0.00775$\pm$0.00072}&\textbf{246.33$\pm$55.39}\\\cline{2-5}
			&CSC-$\ell_1$&0.0104$\pm$0.0012&0.0249$\pm$0.0008&715.44$\pm$93.86\\\cline{2-5}
			&$\alpha$CSC&\textbf{0.00353$\pm$0.00017}&\textbf{0.00766$\pm$0.00052}&1986.08$\pm$262.47\\\cline{2-5}
			&GCSC &\textbf{0.00343$\pm$0.00021}&\textbf{0.00762$\pm$0.00026}&\textbf{238.52$\pm$19.78}\\\hline			
			\multirow{5}{*}{Laplace noise}	
			&CSC-$\ell_2$&0.00835$\pm$0.00042&0.0114$\pm$0.0010&470.13$\pm$ 27.44\\\cline{2-5}
			&CSC-$\ell_1$&\textbf{0.00347$\pm$0.00026}&\textbf{0.00735$\pm$0.00038}&\textbf{358.64$\pm$175.40}\\\cline{2-5}
			&$\alpha$CSC &0.00692$\pm$0.00042&0.00973$\pm$0.00013&2309.94$\pm$724.53\\\cline{2-5}
			&GCSC
			&\textbf{0.00335$\pm$0.00017}&\textbf{0.00732$\pm$0.00020}&\textbf{350.76$\pm$68.33}\\\hline	
			\multirow{5}{*}{alpha-stable noise}			
			&CSC-$\ell_2$&0.0702$\pm$0.0060&0.0840$\pm$0.0091&597.83$\pm$90.70\\\cline{2-5}
			&CSC-$\ell_1$&0.0160$\pm$0.0025&0.0337$\pm$0.0047&476.24$\pm$37.92\\\cline{2-5}
			&$\alpha$CSC&\textbf{0.00416$\pm$0.00039}&\textbf{0.00821$\pm$0.00024}&2198.32$\pm$470.57\\\cline{2-5}
			&GCSC &\textbf{0.00402$\pm$0.00030}&\textbf{0.00815$\pm$0.00041}&\textbf{412.43$\pm$71.22}\\\hline		
			\multirow{5}{*}{zero-mean mixture noise}						
			&CSC-$\ell_2$&0.0321$\pm$0.0007&0.0545$\pm$0.0010&\textbf{344.08$\pm$27.44}\\\cline{2-5}
			&CSC-$\ell_1$&0.0604$\pm$0.0055&0.0849$\pm$0.0059&588.57$\pm$88.89\\\cline{2-5}
			&$\alpha$CSC &0.0114$\pm$0.0002&0.0158$\pm$0.0003&1120.70$\pm$ 463.70\\\cline{2-5}
			&GCSC 
			&\textbf{0.00531$\pm$0.00021}&\textbf{0.00971$\pm$0.00082}&\textbf{336.00$\pm$77.85}\\\hline			
			\multirow{5}{*}{nonzero-mean mixture noise}			
			&CSC-$\ell_2$&0.0732$\pm$0.0015&0.0151$\pm$0.0011&642.03$\pm$84.44\\\cline{2-5}
			&CSC-$\ell_1$&0.0670$\pm$0.0037&0.0130$\pm$0.0013&788.60$\pm$88.89\\\cline{2-5}
			&$\alpha$CSC&0.0667$\pm$0.0002&0.0127$\pm$0.0014&2882.26$\pm$907.28\\\cline{2-5}
			&GCSC &\textbf{0.00556$\pm$0.00024}&\textbf{0.00818$\pm$0.00037}&\textbf{471.40$\pm$87.90}\\\hline											
		\end{tabular}
	\end{center}
	\label{tab:syn_eval}
	\vspace{-5px}
\end{table*}
\setcounter{figure}{2}
\begin{figure*}[htb]
	\small
	\centering
	
	{\includegraphics[width=0.15\columnwidth]
		{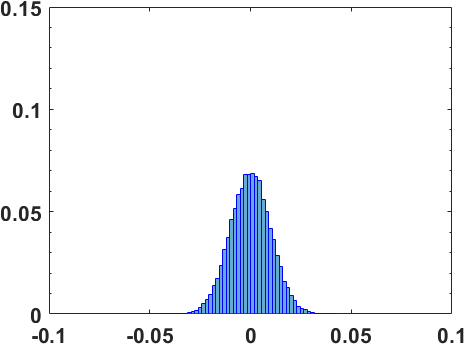}}			
	{\includegraphics[width=0.15\columnwidth]
		{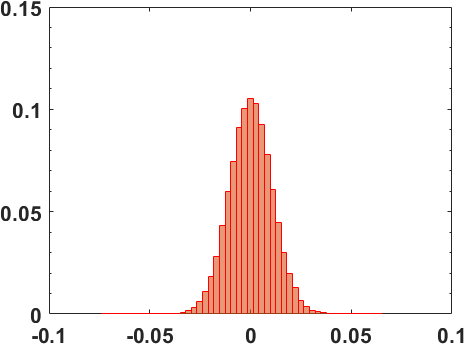}}	
	{\includegraphics[width=0.15\columnwidth]
		{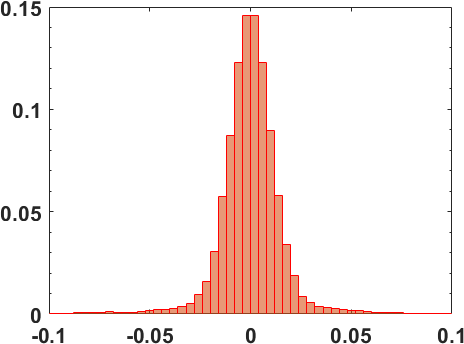}}		
	{\includegraphics[width=0.15\columnwidth]
		{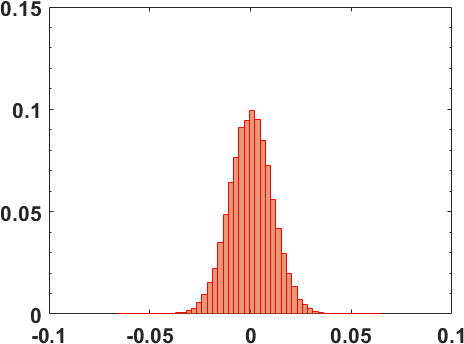}}
	{\includegraphics[width=0.15\columnwidth]
		{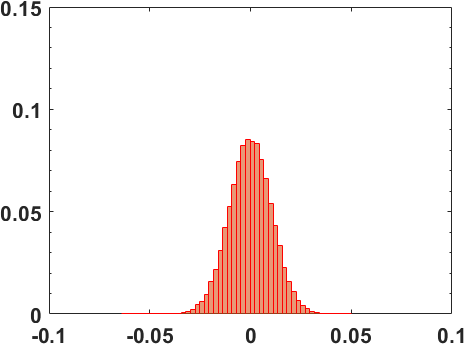}}	
	\vspace{3px}
	
	{\includegraphics[width=0.15\columnwidth]
		{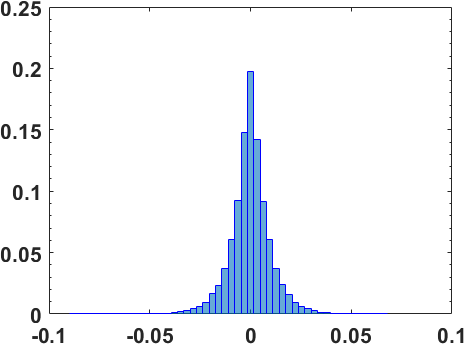}}	
	{\includegraphics[width=0.15\columnwidth]
		{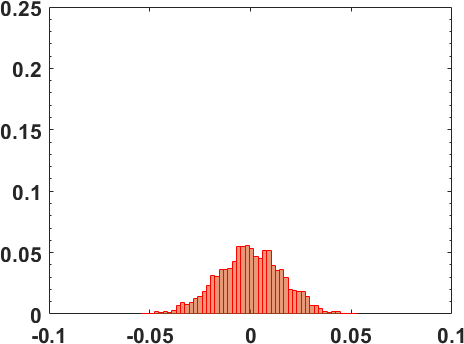}}		
	{\includegraphics[width=0.15\columnwidth]
		{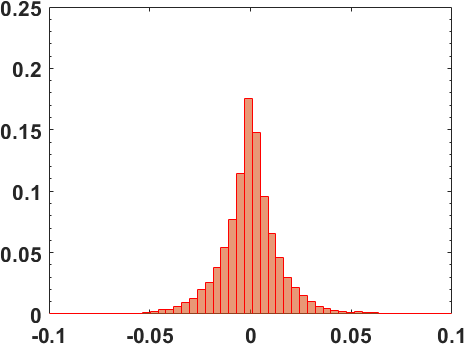}}		
	{\includegraphics[width=0.15\columnwidth]
		{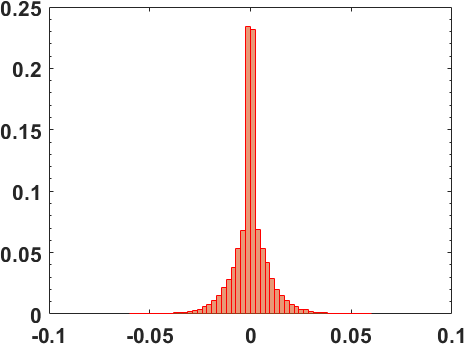}}		
	{\includegraphics[width=0.15\columnwidth]
		{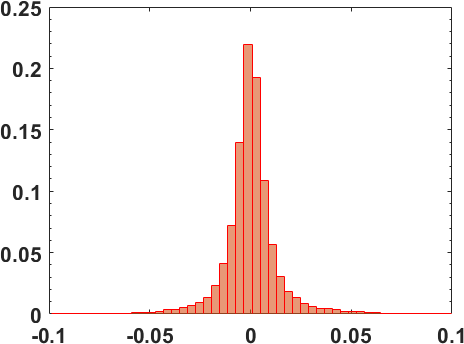}}	
	\vspace{3px}
	
	{\includegraphics[width=0.15\columnwidth]
		{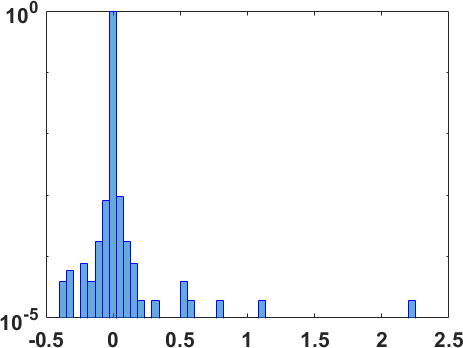}}	
	{\includegraphics[width=0.15\columnwidth]
		{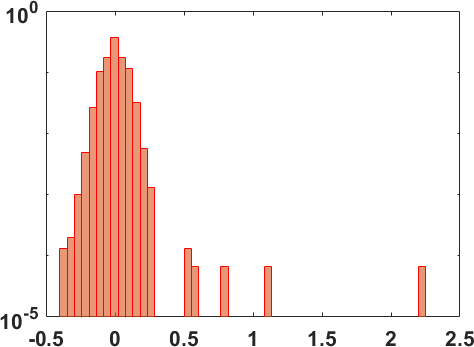}}	
	{\includegraphics[width=0.15\columnwidth]
		{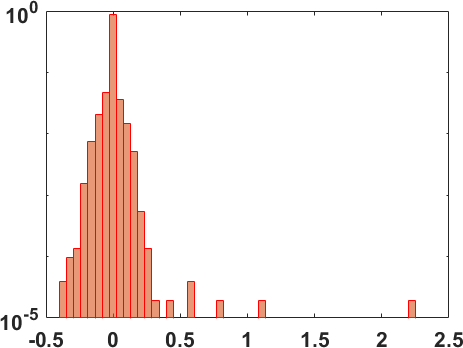}}		
	{\includegraphics[width=0.15\columnwidth]
		{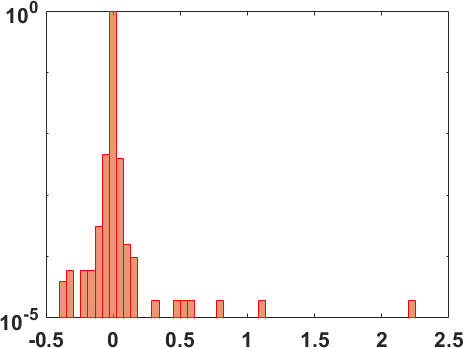}}	
	{\includegraphics[width=0.15\columnwidth]
		{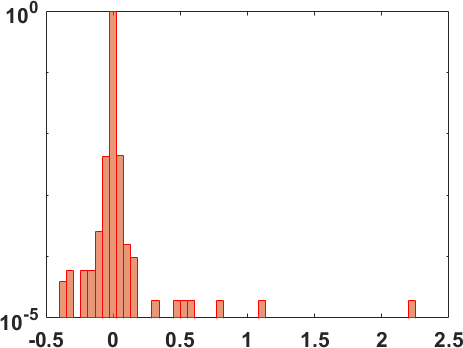}}
	\vspace{3px}
	
	{\includegraphics[width=0.15\columnwidth]
		{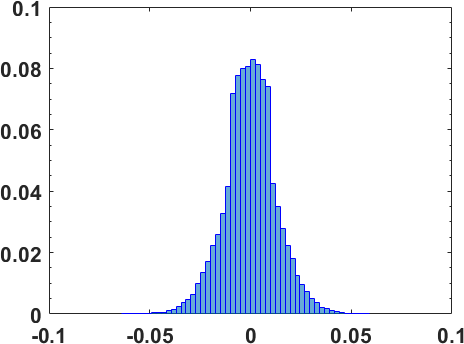}}	
	{\includegraphics[width=0.15\columnwidth]
		{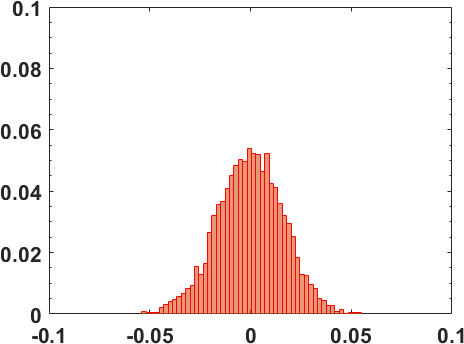}}	
	{\includegraphics[width=0.15\columnwidth]
		{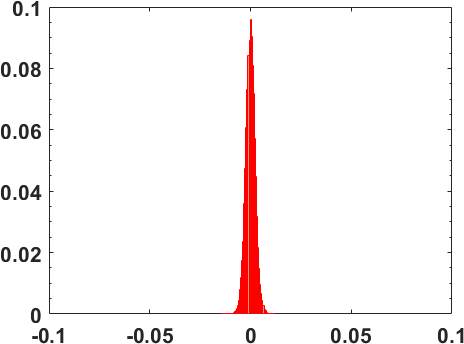}}		
	{\includegraphics[width=0.15\columnwidth]
		{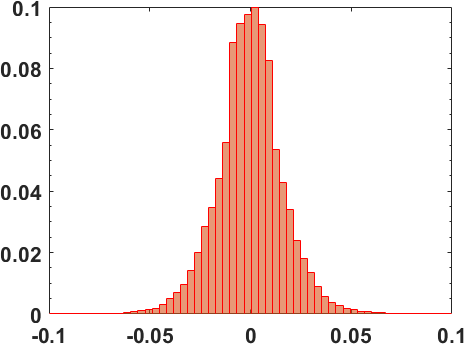}}	
	{\includegraphics[width=0.15\columnwidth]
		{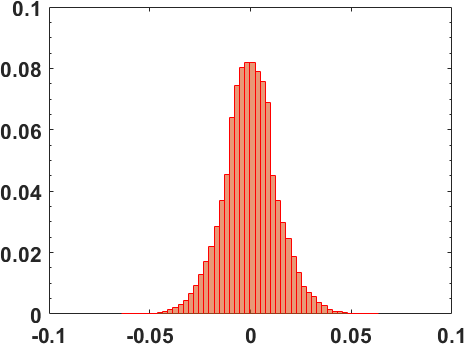}}
	\vspace{3px}

	\subfigure[Ground truth.]
	{\includegraphics[width=0.15\columnwidth]
		{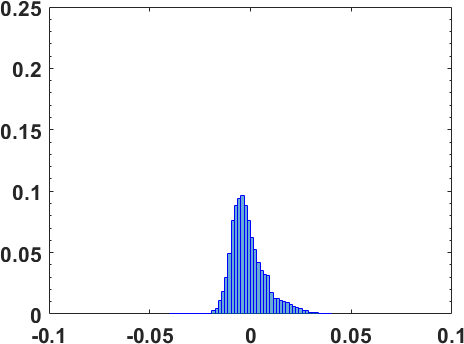}}	
	\subfigure[CSC-$\ell_2$.]
	{\includegraphics[width=0.15\columnwidth]
		{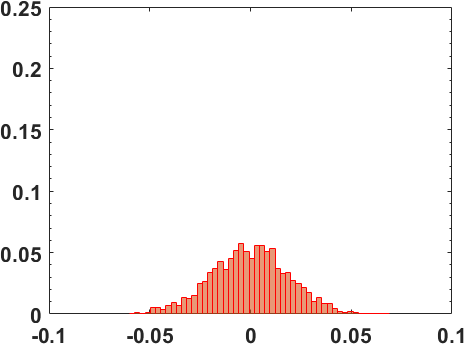}}	
	\subfigure[CSC-$\ell_1$.]
	{\includegraphics[width=0.15\columnwidth]
		{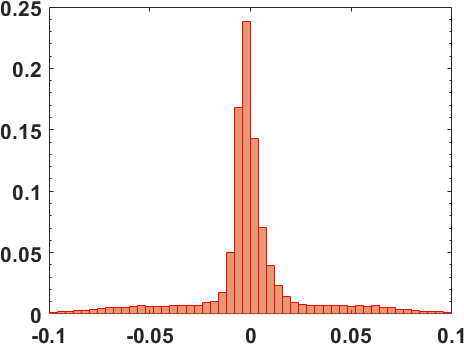}}		
	\subfigure[$\alpha$CSC.]
	{\includegraphics[width=0.15\columnwidth]
		{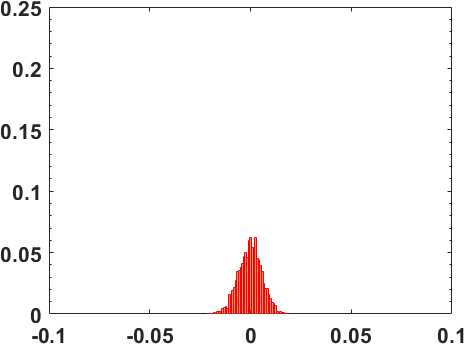}}	
	\subfigure[GCSC.]
	{\includegraphics[width=0.15\columnwidth]
		{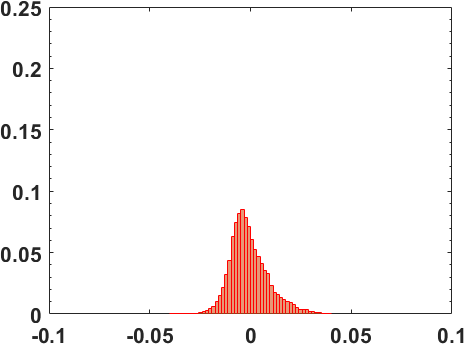}}	
	\vspace{-10px}
	\caption{Noise histograms fitted by the different models in the synthetic experiment. 
		The histograms are normalized by the number of elements.
		For each row, 
		Row 1: Gaussian noise;
		Row 2: Laplace noise;
		Row 3: alpha-stable noise;
		Row 4: zero-mean mixture noise,
		Row 5: nonzero-mean mixture noise.
	}
	
	\label{fig:syn_data_noise_dist}
\end{figure*}

\section{Experiments}
\label{sec:expts}

In this section, we perform experiments on both 
synthetic 
and real-world data sets.
Experiments are performed on a PC with Intel i7 4GHz CPU with 32GB memory.

\subsection{Baseline Algorithms}
\label{sec:expts_setup}

The proposed GCSC is compared with the following CSC 
state-of-the-arts:
\begin{enumerate}
\item CSC-$\ell_2$ \cite{heide2015fast}\footnote{\url{http://www.cs.ubc.ca/labs/imager/tr/2015/FastFlexibleCSC/}},
which models noise by the Gaussian distribution.
\item Alpha-stable CSC ($\alpha$CSC) \cite{jas2017learning}\footnote{\url{https://alphacsc.github.io/}}, 
which uses symmetric 
alpha-stable
distribution to
models noise by setting the parameters of alpha-stable distribution 
$\mathcal{S}(\alpha,\beta,\sigma,\mu)$ (with stability parameter $\alpha$, skewness parameter $\beta$, scale parameter $\sigma$ and position parameter $\mu$)
as $\beta=0,\sigma = \frac{1}{\sqrt{2}}$ and $\mu=\tilde{x}_i(p)$ where $\tilde{x}_i$ is defined as in \eqref{eq:x}.\footnote{In \cite{jas2017learning}, 
$\alpha$ 
is simply set to
1.2.
Here, we choose $\alpha$ by using a validation set, 
which is obtained by
randomly sampling 20\% of the samples.}
\end{enumerate}
As a further baseline, we also compare with 
CSC-$\ell_1$, a CSC variant which models noise by the Laplace distribution.  
It is formulated as the following optimization problem:
\begin{align}
\label{eq:cscl1}
\!\!\min_{\{d_k\}\in\mathcal{D},\{z_{ik}\}}
\sum_{i=1}^{N}
\left( 
\frac{1}{2}
\NML{{x}_i
	\!\!-\!\!
	\sumDZ}{1}
\!\!+\!\! 
\sum_{k=1}^K
\beta\NM{z_{ik}}{1}\right) .\!\!
\end{align}
Details are in Appendix~\ref{app:csc_l1}.

We follow the automatic pruning strategy in \cite{meng2013robust}
to select the number of mixture components $G$ in 
GCSC.
We start with a relatively large $G=10$. At each
EM iteration, 
the relative difference among all Gaussians
are computed:
\begin{equation*}
\sum_{p=1}^{P}
\frac{|\sigma^2_{a}(p)-\sigma^2_{b}(p)|}{\sigma^2_{a}(p)+\sigma^2_{b}(p)},~\forall a,b\in [1,\dots, G].
\end{equation*}
For the Gaussian pair with the smallest relative difference, if this value is small
(less than 0.1), they are merged as 
\begin{eqnarray*}
	\pi_{a} &\leftarrow& \pi_{a}+\pi_{b},
	~
	\mu_{a} \leftarrow \frac{\pi_{a}\mu_{a}+\pi_{b}\mu_{b}}{\pi_{a}+\pi_{b}},\\
	\Sigma_a &\leftarrow& \diag(\sigma^2_{a}(1),\dots,\sigma^2_{a}(P)),
	~\text{where}~\sigma^2_{a}(p) =
	\frac{\pi_{a}\sigma^2_{a}(p)+\pi_{b}\sigma^2_{b}(p)}{\pi_{a}(p)+\pi_{b}(p)}.
\end{eqnarray*}

\subsubsection*{Stopping Criteria}
The optimization problems in 
$\alpha$CSC and GCSC
are solved by the EM algorithm.
We stop the EM iterations  
when the
relative change of
log posterior 
in consecutive iterations
is smaller than ${10}^{-4}$.
In the M-step, 
we stop the updating of weighted CSC (Algorithm~\ref{alg:wcsc} for GCSC, and Algorithm in Appendix B of $\alpha$CSC paper \cite{jas2017learning}) 
if the
relative change of the respective objective value 
is smaller than ${10}^{-4}$.

The optimization problems in 
CSC-$\ell_1$ and CSC-$\ell_2$ are solved by BCD. Alternating minimization 
is stopped when the
relative change of objective value (\eqref{eq:csc} for CSC-$\ell_2$ and 
\eqref{eq:cscl1} for CSC-$\ell_1$) 
in consecutive iterations
is smaller than ${10}^{-4}$.
As for the optimization subproblems of $d_k$'s (given $z_{ik}$'s) 
and $z_{ik}$'s (given $d_k$'s), we stop when the 
relative change of objective value 
is smaller than ${10}^{-4}$.

\begin{figure*}[htb]
	\small
	\centering
	{\includegraphics[width=0.15\columnwidth]
		{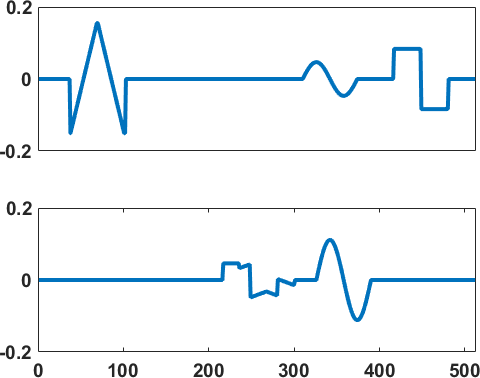}}	
	{\includegraphics[width=0.15\columnwidth]
		{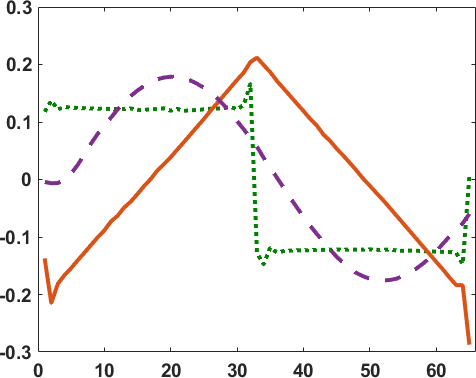}}	
	{\includegraphics[width=0.15\columnwidth]
		{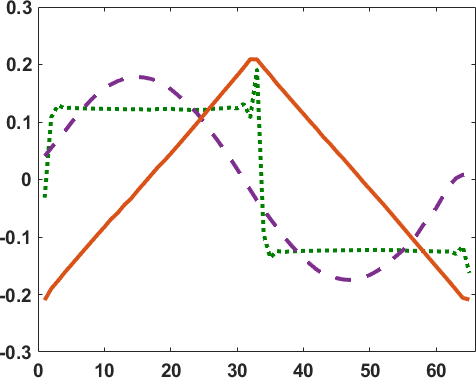}}			
	{\includegraphics[width=0.15\columnwidth]
		{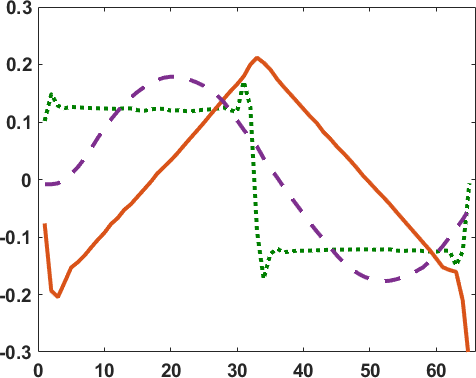}}	
	{\includegraphics[width=0.15\columnwidth]
		{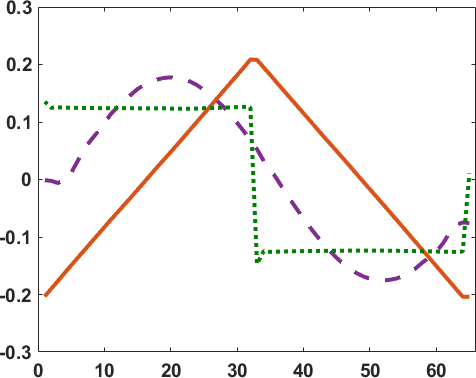}}
	
	{\includegraphics[width=0.15\columnwidth]
		{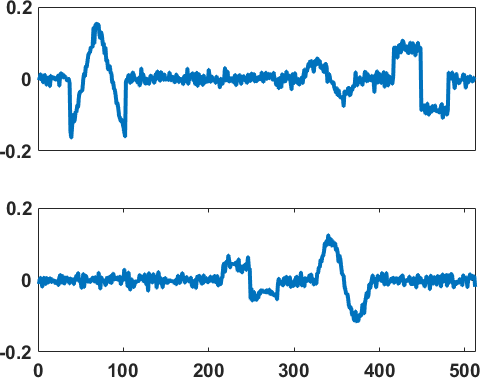}}			
	{\includegraphics[width=0.15\columnwidth]
		{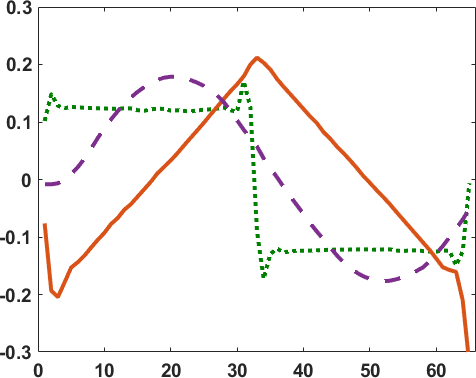}}	
	{\includegraphics[width=0.15\columnwidth]
		{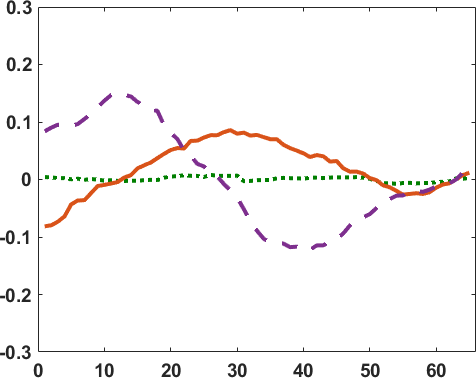}}		
	{\includegraphics[width=0.15\columnwidth]
		{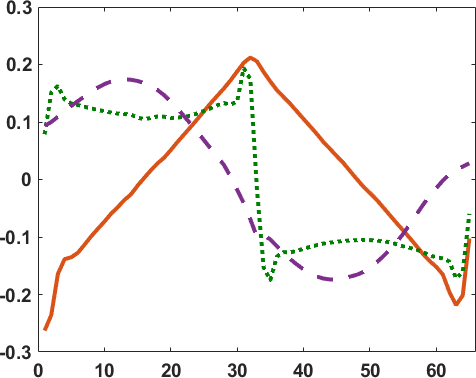}}
	{\includegraphics[width=0.15\columnwidth]
		{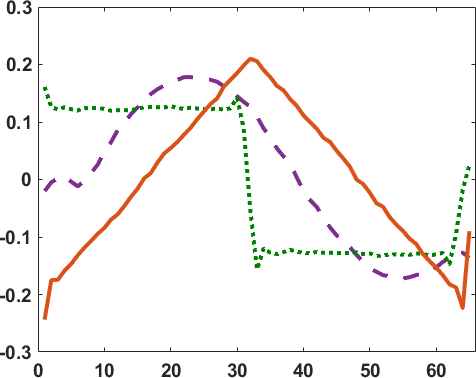}}	
	
	{\includegraphics[width=0.15\columnwidth]
		{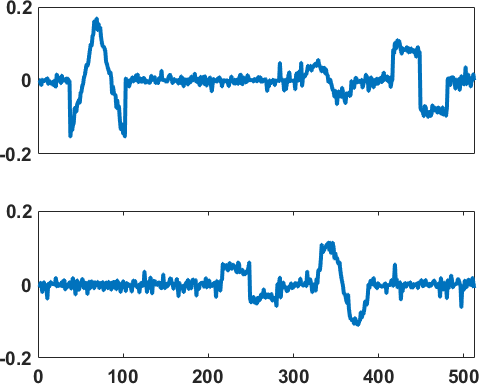}}	
	{\includegraphics[width=0.15\columnwidth]
		{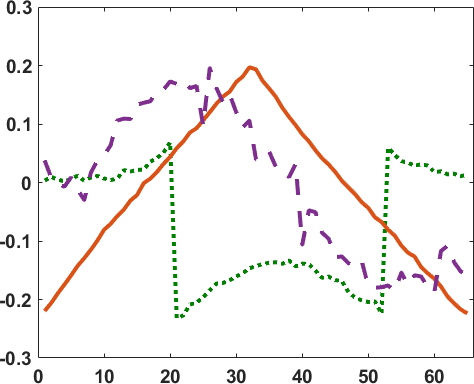}}		
	{\includegraphics[width=0.15\columnwidth]
		{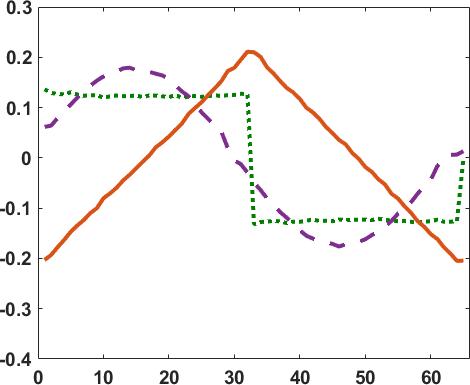}}		
	{\includegraphics[width=0.15\columnwidth]
		{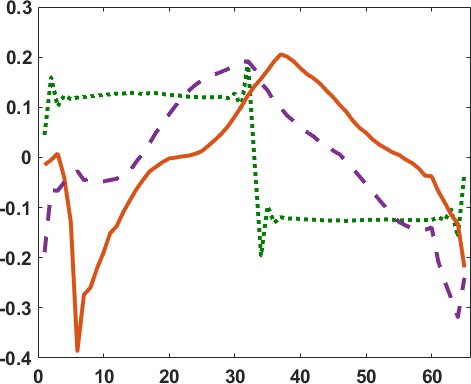}}		
	{\includegraphics[width=0.15\columnwidth]
		{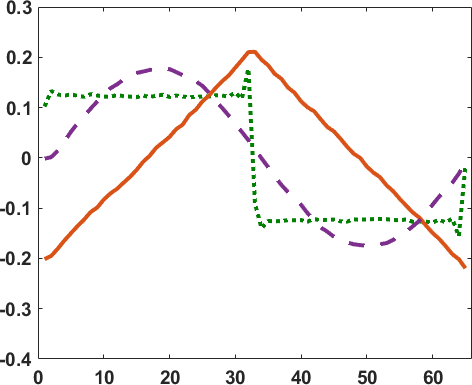}}	
	
	{\includegraphics[width=0.15\columnwidth]
		{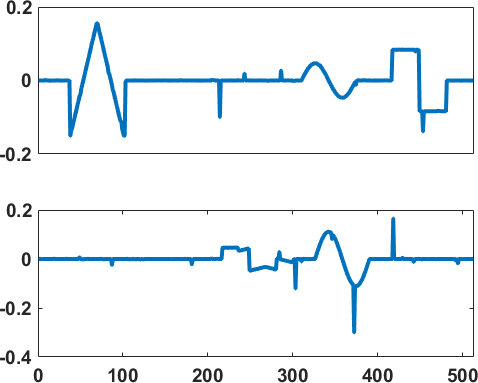}}	
	{\includegraphics[width=0.15\columnwidth]
		{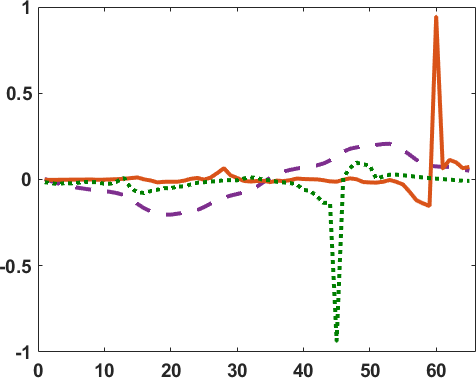}}	
	{\includegraphics[width=0.15\columnwidth]
		{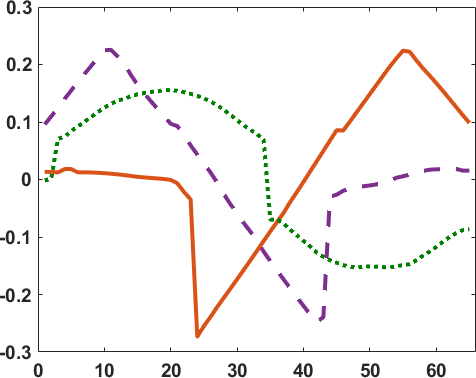}}		
	{\includegraphics[width=0.15\columnwidth]
		{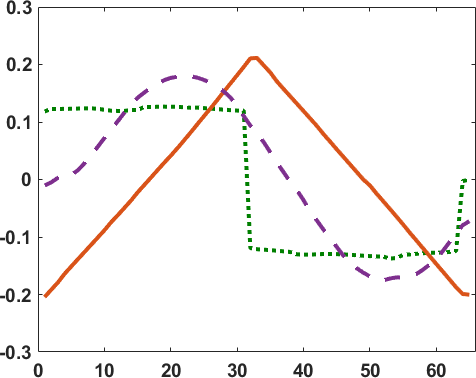}}	
	{\includegraphics[width=0.15\columnwidth]
		{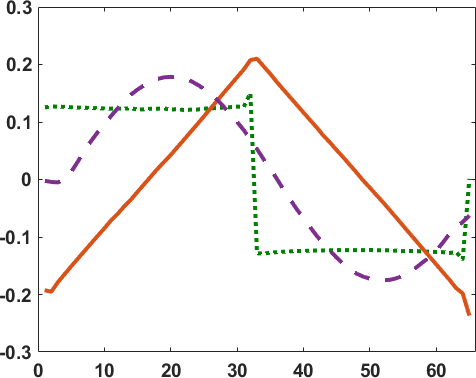}}	
	
	{\includegraphics[width=0.15\columnwidth]
		{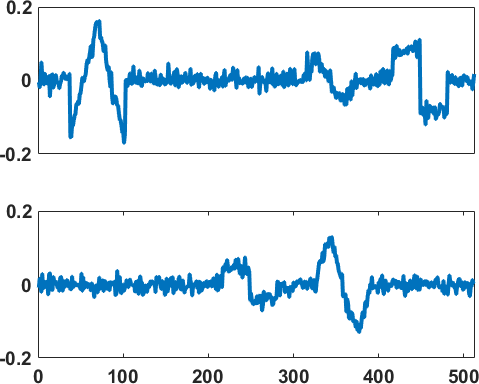}}	
	{\includegraphics[width=0.15\columnwidth]
		{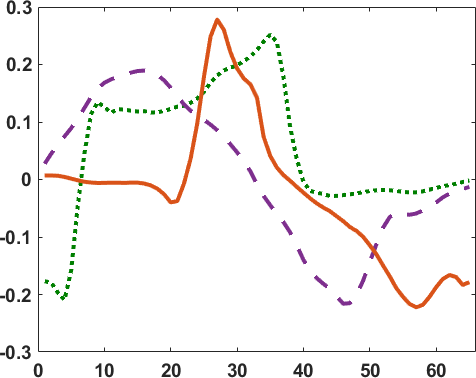}}	
	{\includegraphics[width=0.15\columnwidth]
		{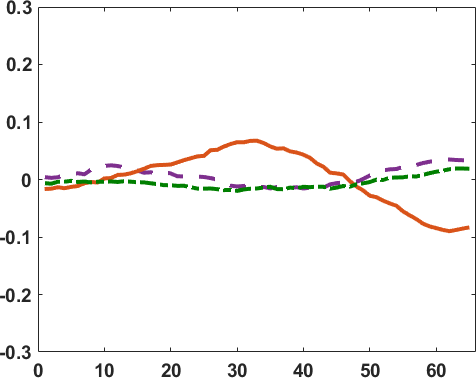}}		
	{\includegraphics[width=0.15\columnwidth]
		{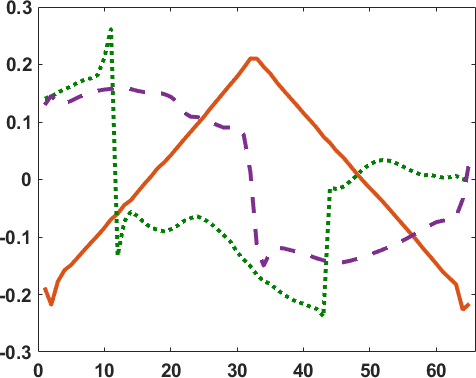}}	
	{\includegraphics[width=0.15\columnwidth]
		{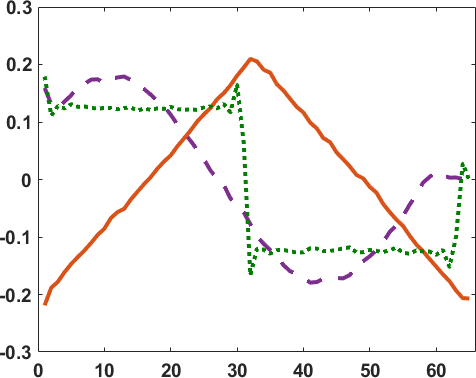}}	
	
	\subfigure[$x_i$'s.]
	{\includegraphics[width=0.15\columnwidth]
		{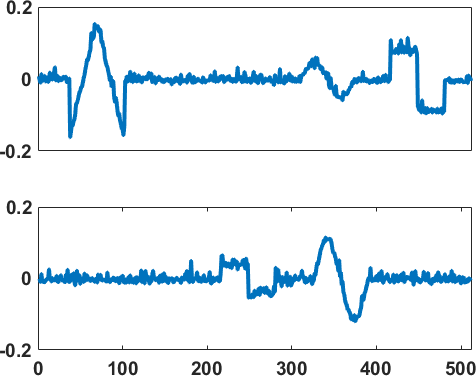}}	
	\subfigure[CSC-$\ell_2$.]
	{\includegraphics[width=0.15\columnwidth]
		{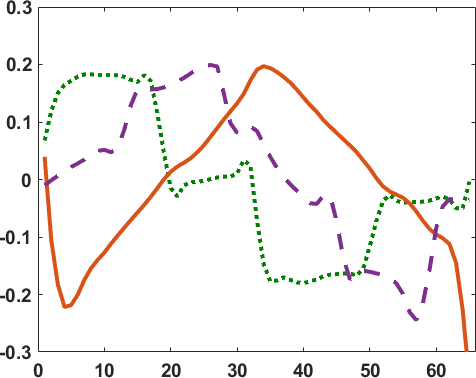}}	
	\subfigure[CSC-$\ell_1$.]
	{\includegraphics[width=0.15\columnwidth]
		{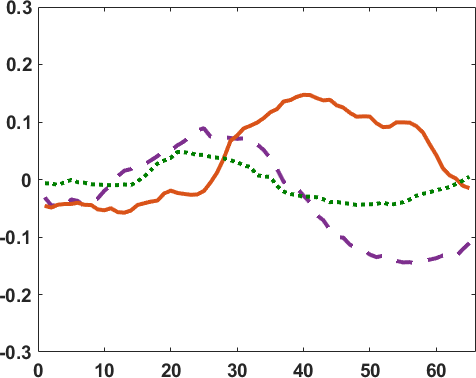}}		
	\subfigure[$\alpha$CSC.]
	{\includegraphics[width=0.15\columnwidth]
		{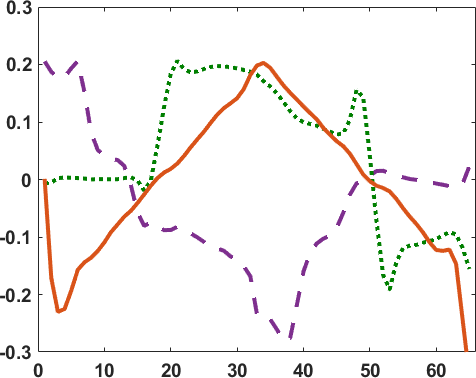}}	
	\subfigure[GCSC.]
	{\includegraphics[width=0.15\columnwidth]
		{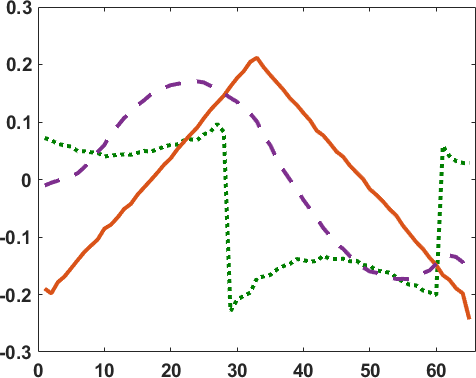}}	
	
	\vspace{-10px}
	
	\caption{Filters obtained by the different models on synthetic data. 
		Row 1: no noise;
		Row 2: Gaussian noise;
		Row 3: Laplace noise;
		Row 4: alpha-stable noise;
		Row 5: zero-mean mixture noise;
		Row 6: nonzero-mean mixture noise.
	}
	\label{fig:syn_filters}
\end{figure*}


\subsection{Synthetic Data}
\label{sec:expts_syn}

In this experiment, we first demonstrate the performance on 
synthetic data.
Following \cite{jas2017learning}, we use 
$K=3$ 
filters $d_k$'s 
(triangle, 
square,
and sine),
each of length $M=65$
(Figure~\ref{fig:syn_data_filter}).
Each $d_k$ is normalized to have zero mean and unit variance.
Each $z_{ik}$
has only one nonzero entry, whose magnitude is uniformly drawn from $[0,1]$
(Figure~\ref{fig:syn_data_code}).
$N=100$ 
clean samples, each of length
$P=512$,
are generated as:
${x}^\text{clean}_i=\sumDZ$ (Figure~\ref{fig:syn_data_data}).

\setcounter{figure}{1}
\begin{figure}[ht]
\centering		
\subfigure[The 3 $d_k$'s.\label{fig:syn_data_filter}]	
{\includegraphics[width=0.15\textwidth]
{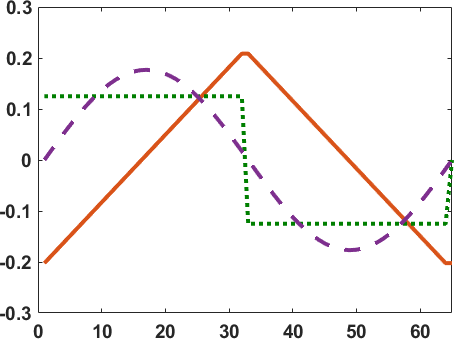}}		
\subfigure[The 3 $z_{ik}$'s.\label{fig:syn_data_code}]
	{\includegraphics[width=0.15\textwidth]
		{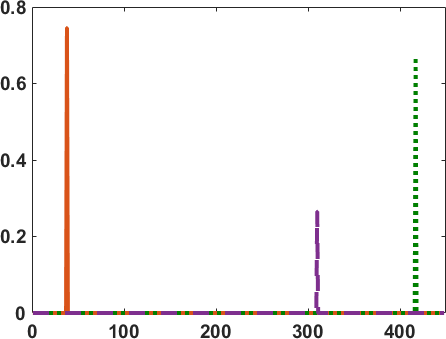}}	
	\subfigure[${x}^\text{clean}_i$.\label{fig:syn_data_data}]	
	{\includegraphics[width=0.15\textwidth]
		{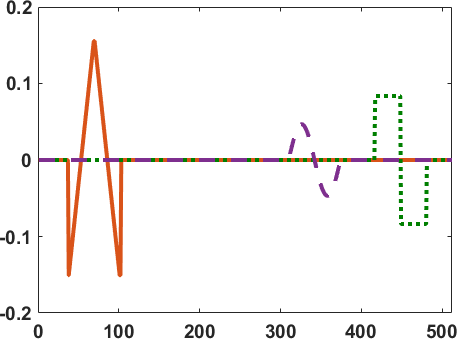}}			
	\vspace{-.1in}
\caption{Example of synthesizing a sample.}
	\label{fig:syn_data}
\end{figure}

Noise  is then 
added to generate
observations
${x}_i$'s. 
Following \cite{meng2013robust},
different types of noise are considered
(Table~\ref{tab:syn_noise}).
The alpha-stable noise we considered is Cauchy distribution,
which is one representative symmetric alpha-stable distribution apart from Gaussian and should be modeled well by $\alpha$CSC and GCSC. 

\setcounter{table}{1}	
\begin{table}[htbp]
\caption{Types of noise added to synthetic data. 
}
\centering
\vspace{-5px}
\begin{tabular}{c|c|c}
\hline
noise & distribution&  SNR (dB) \\ \hline
Gaussian & $\mathcal{N}(0,0.01^2)$ &        13.04        \\ \hline
Laplace & $\mathcal{L}(0,0.01)$ &        13.03        \\ \hline
alpha-stable & $\mathcal{S}(1,0,0.01^2,0)$  &       10.43        \\ \hline
\multirow{3}{*}{zero-mean mixture}&20\% from $\mathcal{U}(-0.01,0.01)$,&\multirow{3}{*}{10.98}\\
& 20\% from $\mathcal{N}(0,0.01^2)$,&\\
&60\% from
$\mathcal{N}(0,0.015^2)$&\\\hline
\multirow{3}{*}{nonzero-mean mixture}&20\% from $\mathcal{U}(-0.01,0.01)$,&\multirow{3}{*}{14.16}\\
& 20\% from $\mathcal{N}(0.01,0.01^2)$,&\\
&60\%  from $\mathcal{N}(-0.005,0.005^2)$&\\\hline
	\end{tabular}
	\label{tab:syn_noise}
\end{table}

\setcounter{figure}{4} 
\begin{figure*}[ht]
	\centering
	
	\subfigure[$x^{\text{clean}}_i$.]{\includegraphics[width=0.15\columnwidth]{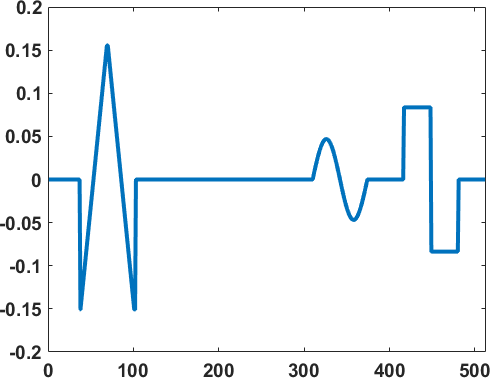}}
	\subfigure[$x_i$.]{\includegraphics[width=0.15\columnwidth]{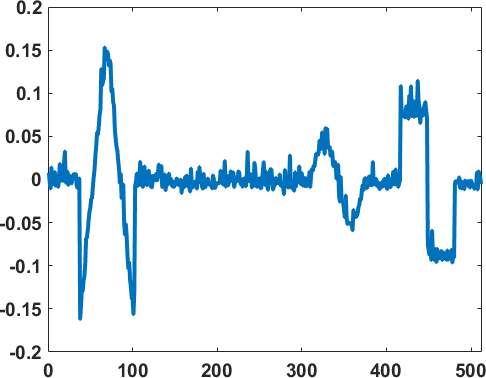}}
	\subfigure[CSC-$\ell_2$.]{\includegraphics[width=0.15\columnwidth]{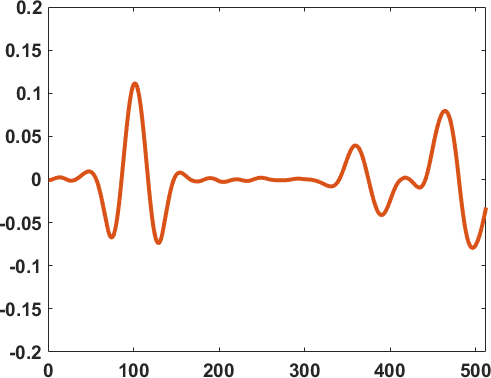}}
	\subfigure[CSC-$\ell_1$.]{\includegraphics[width=0.15\columnwidth]{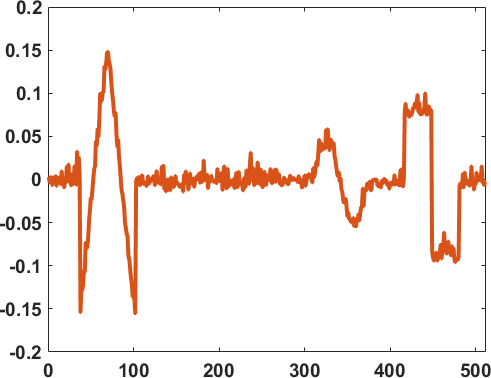}}
	\subfigure[$\alpha$CSC.]{\includegraphics[width=0.15\columnwidth]{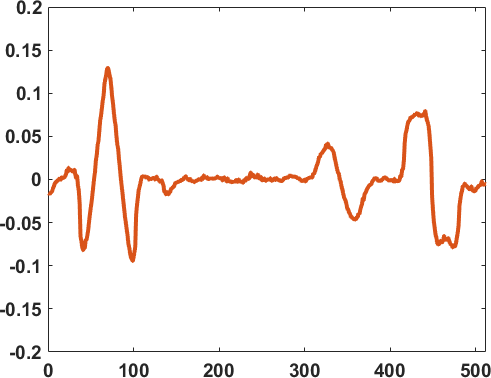}}
	\subfigure[GCSC.]{\includegraphics[width=0.15\columnwidth]{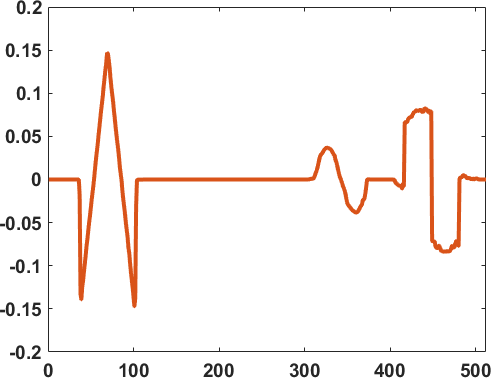}}	
	\vspace{-10px}			
	\caption{
		Reconstructions obtained by the different models on synthetic data with nonzero-mean mixture noise. 
	}
	\label{fig:syn_rec}
\end{figure*}

\subsubsection{Quantitative Evaluation}
Following \cite{meng2013robust,chen2016robust}, 
performance is evaluated by 
the mean absolute error (MAE) and root mean squared error (RMSE): 
\begin{eqnarray}\notag
\text{MAE}=&\frac{1}{NP}\sum_{i=1}^{N}\NM{{x}^\text{clean}_{i}-\tilde{x}_{i}}{1},
\\\notag
\text{RMSE}=&\sqrt{\frac{1}{NP}\sum_{i=1}^{N}\NM{{x}^\text{clean}_{i}-\tilde{x}_{i}}{2}^2},
\end{eqnarray}
where
$\tilde{x}_i= \sum_{k=1}^{K}d_k*z_{ik}$ 
is the reconstruction 
based on the obtained $d_k$'s and $z_{ik}$'s.
Results are averaged over five runs with different initializations of $d_k$'s and $z_{ik}$'s.

Results are shown in
Table~\ref{tab:syn_eval}.
When there is no noise, all methods obtain MAE and RMSE in the same order of magnitude.
When there is noise, intuitively, 
the model
whose underlying noise assumption matches the actual noise distribution
will perform the best.
Empirically, GCSC is the best or comparable with the best method on all types of noise.

As for time,
CSC-$\ell_2$ and GCSC are the fastest
in general.
CSC-$\ell_1$ is slower as it needs more auxiliary variables in ADMM to handle the nonsmooth $\ell_1$ loss (details are in Appendix~\ref{app:csc_l1}).
 $\alpha$CSC 
is the slowest, as it performs
in the spatial domain which
is costly. Moreover,  it requires expensive Markov chain Monte Carlo (MCMC) in its E-step.

\subsubsection{Visual Comparison}
Figure~\ref{fig:syn_data_noise_dist} compares the ground truth
noise with those fitted by the models.
As can be seen, 
GCSC models 
the noise well for all types of noise, 
while the other methods only model the noise well when 
its underlying noise distribution matches the actual noise.

Figure~\ref{fig:syn_filters}
shows the learned filters.
As can be seen, GCSC can 
recover the underlying filters more reliably.
Figure~\ref{fig:syn_rec} further shows the reconstructions on synthetic data with nonzero-mean noise.
GCSC is the only one that denoises well and recover the underlying the clean data.

\subsubsection{Solving \eqref{eq:wcsc_F}: niAPG vs BCD}
We first 
consider
solving \eqref{eq:wcsc_F} in the M-step of one EM iteration.
The details of BCD solver are in Appendix~\ref{app:wcsc_bcd}.
Figure~\ref{fig:niapg_bcd} shows convergence of solving \eqref{eq:wcsc_F} with time on synthetic data with nonzero-mean noise.
As shown, niAPG converges much faster than BCD.
It rapidly starts to reduce objective and converges to a smaller objective. 
Thus, using niAPG to solve the (weighted) CSC problem is a more efficient choice.
\setcounter{figure}{5} 
\begin{figure}[ht]
	\centering
	\includegraphics[width=0.3\columnwidth]{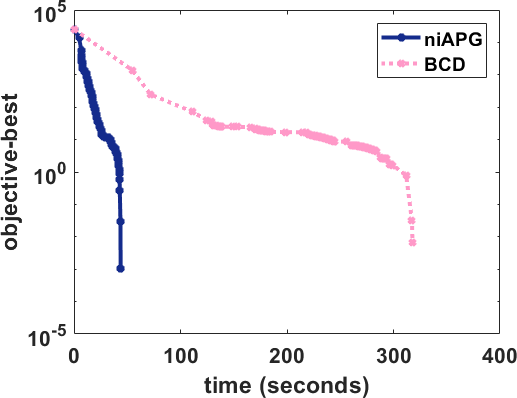}
	\caption{Convergence of niAPG and BCD on solving \eqref{eq:wcsc_F}.}
	\label{fig:niapg_bcd}
\end{figure}
Further, we show the performance of the whole GCSC with GMM loss using different solvers for \eqref{eq:wcsc_F} in the M-step in Table~\ref{tab:syn_solver}. 
Although the two solvers obtain similar MAE and RMSE, BCD solver takes much longer time than niAPG solver. 

\setcounter{table}{3}
\begin{table}[htb]
	\caption{ Performance of GCSC with different solvers for \eqref{eq:wcsc_F} on the synthetic data with nonzero-mean mixture noise.
}
	\vspace{-10px}
	\begin{center}
		\begin{tabular}{c|c|c|c}
			\hline
			&MAE&RMSE&time (seconds)\\ \hline
			BCD&\textbf{0.00557$\pm$0.00031}&\textbf{0.00821$\pm$0.00044}&2562.37$\pm$400.11\\ niAPG&\textbf{0.00556$\pm$0.00024}&\textbf{0.00818$\pm$0.00037}&\textbf{471.40$\pm$87.90}\\ \hline		
		\end{tabular}
		\vspace{-5px}
	\end{center}
	\label{tab:syn_solver}
\end{table}


\subsection{Local Field Potential Data}
\label{sec:expts_lfp}

\setcounter{figure}{6} 
\begin{figure*}[ht]
	\centering				
	{\includegraphics[width=0.165\textwidth]
		{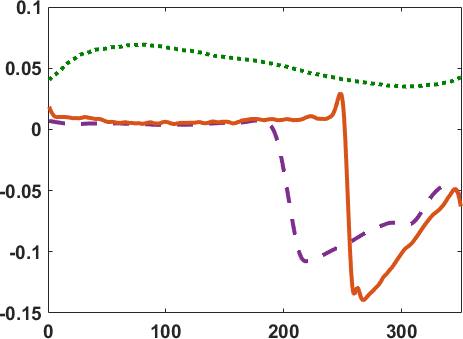}}
	{\includegraphics[width=0.165\textwidth]
		{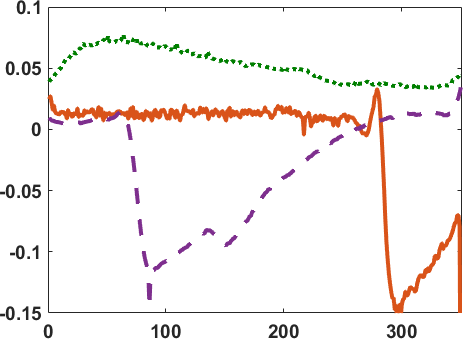}}	
	{\includegraphics[width=0.165\textwidth]
		{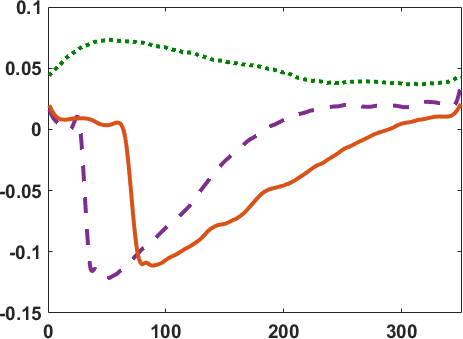}}
	{\includegraphics[width=0.165\textwidth]
		{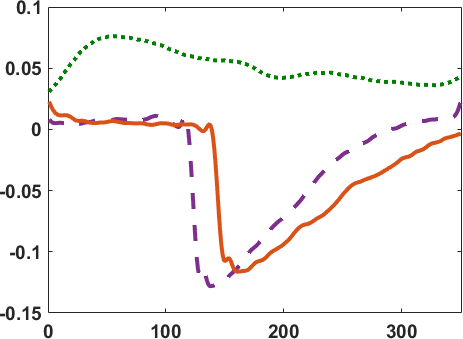}}	
			
	\subfigure[CSC-$\ell_2$.]{\includegraphics[width=0.165\textwidth]
		{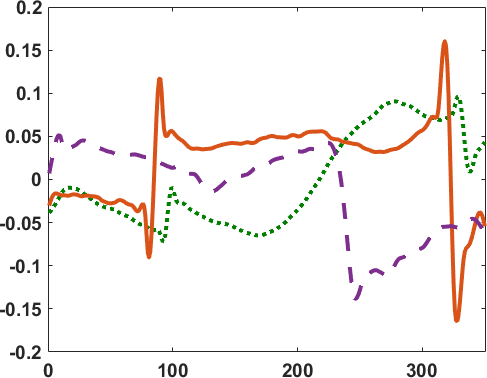}}
	\subfigure[CSC-$\ell_1$.]{\includegraphics[width=0.165\textwidth]
		{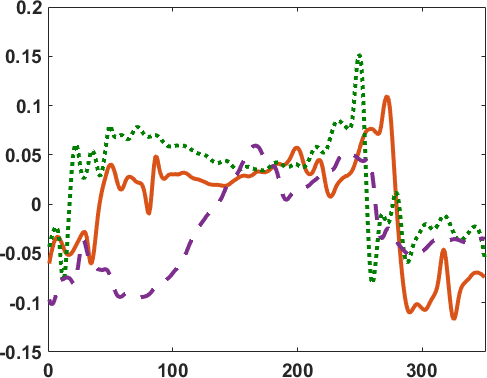}}	
	\subfigure[$\alpha$CSC.]{\includegraphics[width=0.165\textwidth]
		{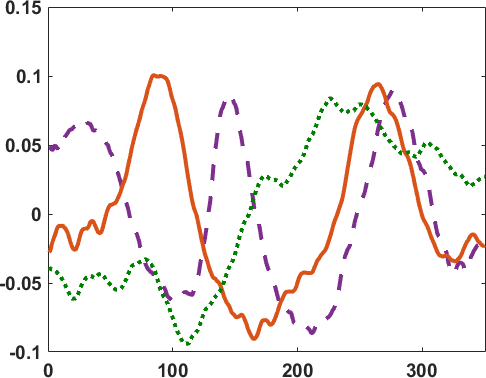}}
	\subfigure[GCSC.]{\includegraphics[width=0.165\textwidth]
		{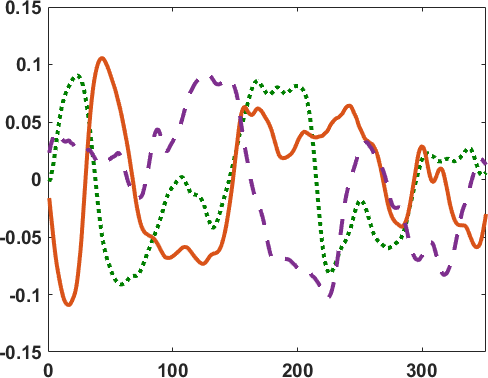}}	
	\vspace{-.1in}
	\caption{
		Filters obtained by the different models on \textit{LFP-cortical} (top) and \textit{LFP-striatal} (bottom).
	}
	\label{fig:lfp_atom}
	\vspace{-10px}
\end{figure*}

In this section, experiments are performed on
two real local field potential (LFP) data sets from \cite{jas2017learning}.
LFP is an electrophysiological signal recording the collective activities of a group of nearby neurons. It is closely related to
cognitive mechanisms such as attention, high-level visual processing and motor control. 
The first 
signal 
(\textit{LFP-cortical}) is 
recorded in the rat cortex 
\cite{hitziger2017adaptive},
while the second one
(\textit{LFP-striatal})
is recorded in the rat striatum 
\cite{dallerac2017updating}. 
Figure~\ref{fig:lfp_data} shows samples from these two data sets.
Note that \textit{LFP-striatal} contains heavier artifact as shown in the local segment.
Following \cite{jas2017learning}, we 
extract $N=100$ non-overlapping segments,
each of length $P=2500$, from each data set.
The other preprocessing steps and parameter setting 
($K=3$ and $M=350$)
are the same as in \cite{jas2017learning}.

Figure~\ref{fig:lfp_atom} shows the learned filters. 
As there is no ground truth, we can only evaluate the results qualitatively. 
For \textit{LFP cortical}, the learned filters are similar to the local regions in segments.
As for \textit{LFP striatal}, severe artifacts contaminate the filters learned by  CSC-$\ell_1$ and CSC-$\ell_2$, 
but do not prevent  
GCSC and $\alpha$CSC from learning  
filters similar in shape to the clean part of the segments. 
We also compare the time in Table~\ref{tab:lfp_time}. On both \textit{LFP-cortical} and \textit{LFP striatal}, GCSC is the fastest.

\begin{figure}[ht]
	\centering		
	\subfigure[whole signal.]
	{\includegraphics[width=0.225\textwidth]
		{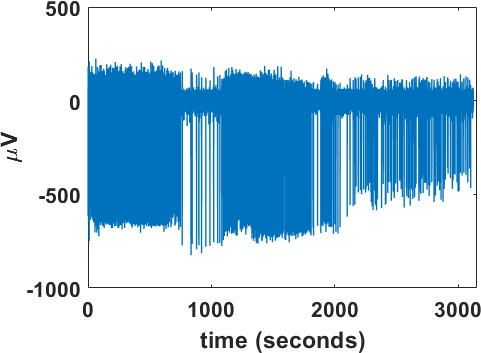}}	
	\subfigure[an expanded segment.]
	{\includegraphics[width=0.225\textwidth]
		{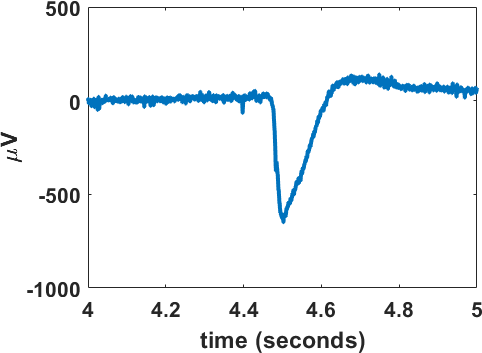}}		
	\subfigure[whole signal.]
	{\includegraphics[width=0.225\textwidth]
		{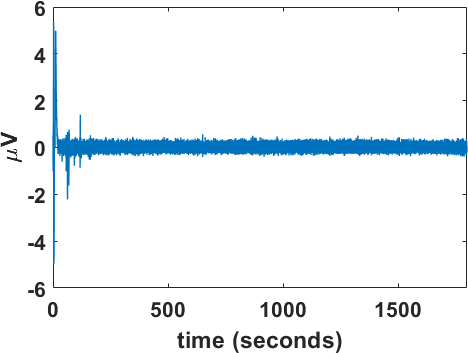}}	
	\subfigure[an expanded segment.]
	{\includegraphics[width=0.225\textwidth]
		{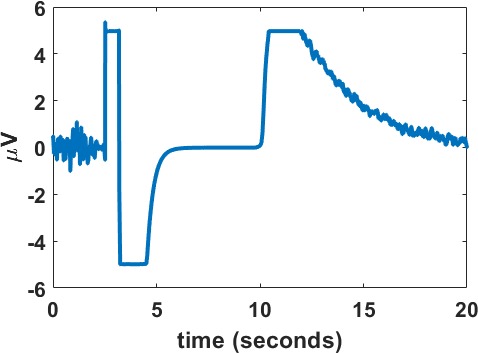}}			
	\vspace{-.1in}
	\caption{Example samples
		from \textit{LFP-cortical} (a,b) and \textit{LFP-striatal} (c,d),
		respectively. 
	}
	\label{fig:lfp_data}
\end{figure}  

 \begin{table}[htp]
	\caption{Timing results (seconds) on the \textit{LFP} data.}
	\vspace{-10px}
	\begin{center}
		\begin{tabular}{c|c|c}
			\hline
			&\textit{LFP-cortical}&\textit{LFP-striatal}\\\hline
			CSC-$\ell_2$&721.71$\pm$47.21&802.37$\pm$51.12\\\hline
			CSC-$\ell_1$&737.62$\pm$68.88&783.16$\pm$76.10\\\hline
			$\alpha$CSC&2919.33$\pm$290.77&3004.87$\pm$320.81\\\hline
			GCSC&\textbf{607.23$\pm$57.56}&\textbf{611.24$\pm$69.91}\\\hline	
		\end{tabular}
	\vspace{-5px}
	\end{center}
	\label{tab:lfp_time}
\end{table}


\subsection{Retinal Image Data}
\label{sec:expts_retina}

In this section, we perform vessel segmentation via pixelwise classification of retinal image data sets. 
The pixel on the retinal vessels is classified as 1 while the pixel on the background is classified as 0.
Two popular 
retinal image 
data sets, \textit{DRIVE} \cite{staal2004ridge} and \textit{STARE} \cite{hoover2000locating} obtained from 
\cite{becker2013supervised}, are used.
\textit{DRIVE} contains 
40 images of size $584\times 565$ 
and
\textit{STARE} contains 20 images of size
$605\times 700$.
Training and testing images are split in half as in \cite{becker2013supervised}. 
Both data sets are provided with 
manual segmentation results from
two experts.
Following \cite{sironi2015learning,becker2013supervised}, we use the first expert's segmentation as ground truth.

The proposed GCSC is compared with 
CSC-$\ell_2$, CSC-$\ell_1$ and $\alpha$CSC 
(all with $K=50$ and $M=11\times 11$).
Following \cite{becker2013supervised},
each pixel, represented by the $K$ learned codes,
is classified using gradient boosting \cite{friedman2001greedy} 
with 500 weak learners.
From each image,
15,000 vessel pixels are sampled
as positive, and 15,000 background pixels are sampled as negative.
As further baselines,
we compare with 
the provided second expert's
manual segmentation 
(denoted ``Expert") and
the state-of-the-art handcrafted
multi-scale Hessian filter (denoted
``Hessian")\footnote{\url{https://www.mathworks.com/matlabcentral/fileexchange/63171-jerman-enhancement-filter}}
\cite{jerman2016enhancement}.
The experiment is repeated five times.

Figure~\ref{fig:prroc_curve} shows
the 
Receiver Operating Characteristic (ROC) 
and 
Precision-Recall (PR) 
curves.
Table~\ref{tab:vessel_eval} shows 
the corresponding Area Under ROC Curve
(AUC)
and the best
F-score. 
As can be seen, GCSC outperforms all the other methods. 
$\alpha$CSC performs slightly better than CSC-$\ell_1$ and CSC-$\ell_2$.
The multi-scale Hessian filter 
is much worse.
The classification performance of
Expert is low,
which is also noted in \cite{sironi2015learning}.

Figures~\ref{fig:vessel_drive} and \ref{fig:vessel_stare}
show the segmentation results from a test image from \textit{DRIVE} and
\textit{STARE}, respectively.
As can be seen, 
the segmentation results produced by $\alpha$CSC, CSC-$\ell_2$ and CSC-$\ell_1$ 
are still noisy. The
Hessian filter shows clearer vessels, but enlarges the pupil and shrinks some tiny vessels.
In contrast, GCSC obtains cleaner segmented vessels.

\begin{table}[htb]
	\centering
	\small
	
	\caption{ Performance on the retinal image data sets. }
	\vspace{-5px}
	\begin{tabular}{ c  c| c|c }
		\hline
		
		&      & AUC             & best F-score      \\ \hline
	\multirow{6}{*}{\textit{DRIVE}} 	&Expert & -               & 0.8935  $\pm$ 0.0000                     \\ \cline{2-4}
		&Hessian      & 0.7314 $\pm$ 0.0000          & 0.8755  $\pm$ 0.0000                    \\ \cline{2-4}
		&CSC-$\ell_2$ & 0.9044 $\pm$ 0.0067         & 0.9806 $\pm$ 0.0065                     \\ \cline{2-4}
		&CSC-$\ell_1$ & 0.9383 $\pm$ 0.0071         & 0.9821  $\pm$ 0.0063                    \\ \cline{2-4}
		&$\alpha$CSC & 0.9401 $\pm$ 0.0051         & 0.9850  $\pm$ 0.0064                    \\ \cline{2-4}
		&GCSC      & \textbf{0.9504$\pm$ 0.0048} & \textbf{0.9969 $\pm$ 0.0066}    \\ \hline
		\multirow{6}{*}{\textit{STARE}}
		&Expert &        -        & 0.7790 $\pm$ 0.0000           \\ \cline{2-4}
		&Hessian      &     0.6623 $\pm$ 0.0000      & 0.8495  $\pm$ 0.0000          \\ \cline{2-4}
		&CSC-$\ell_2$ &     0.9033 $\pm$ 0.0089     & 0.9838 $\pm$ 0.0080          \\ \cline{2-4}
		&CSC-$\ell_1$ &     0.8964 $\pm$ 0.0087     & 0.9757$\pm$ 0.0073           \\ \cline{2-4}
		&$\alpha$CSC & 0.9101 $\pm$ 0.0056         & 0.9907  $\pm$ 0.0062                    \\ \cline{2-4}
		&GCSC      & \textbf{0.9203$\pm$ 0.0066} & \textbf{0.9999$\pm$ 0.0065}  \\ \hline		
	\end{tabular}
	\vspace{-5px}
	\label{tab:vessel_eval}
\end{table}

\begin{figure}[ht]
	\centering
	\subfigure[\textit{DRIVE}.  \label{fig:drive_roc}]{\includegraphics[width=0.225\columnwidth]{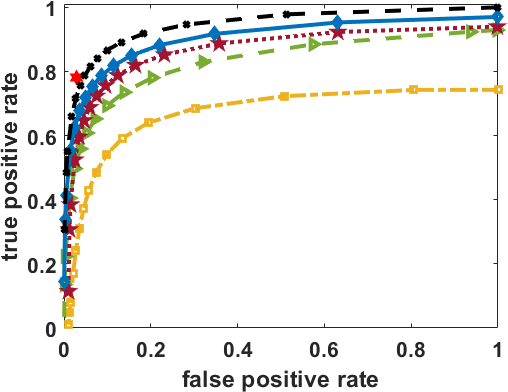}}
	\subfigure[\textit{STARE}.  \label{fig:stare_roc}]{\includegraphics[width=0.225\columnwidth]{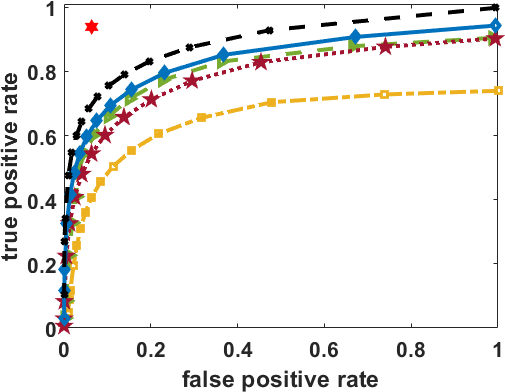}}
	\subfigure[\textit{DRIVE}.  \label{fig:drive_pr}]{\includegraphics[width=0.225\columnwidth]{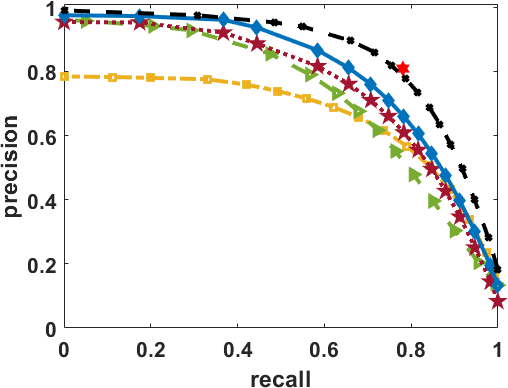}}
	\subfigure[\textit{STARE}.  \label{fig:stare_pr}]{\includegraphics[width=0.225\columnwidth]{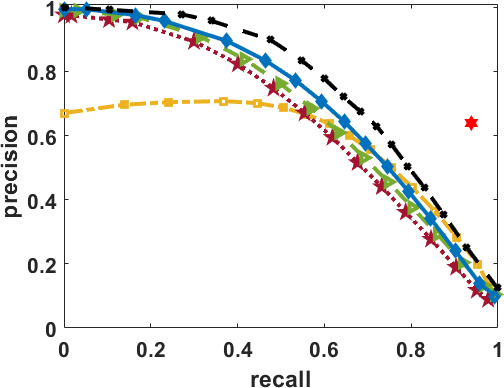}}
	\includegraphics[width=0.45\columnwidth]{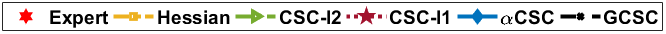}				
	\caption{
		ROC (a,b) and 
		PR (c,d) curves on the retinal image data sets. }
	\label{fig:prroc_curve}
	\vspace{-10px}
\end{figure}

\begin{figure*}[htb]
	\centering		
	\subfigure[Image.
	\label{fig:drive_ori_img}]{\includegraphics[width=0.24\textwidth]
		{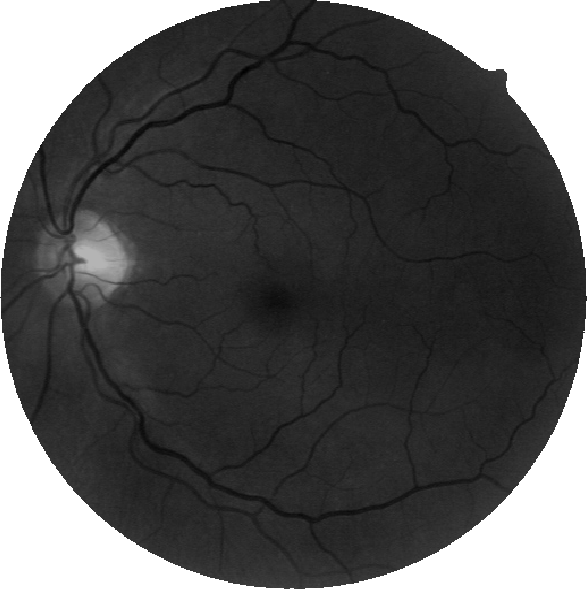}}
	\subfigure[Ground truth.
	\label{fig:drive_gt}]{\includegraphics[width=0.24\textwidth]
		{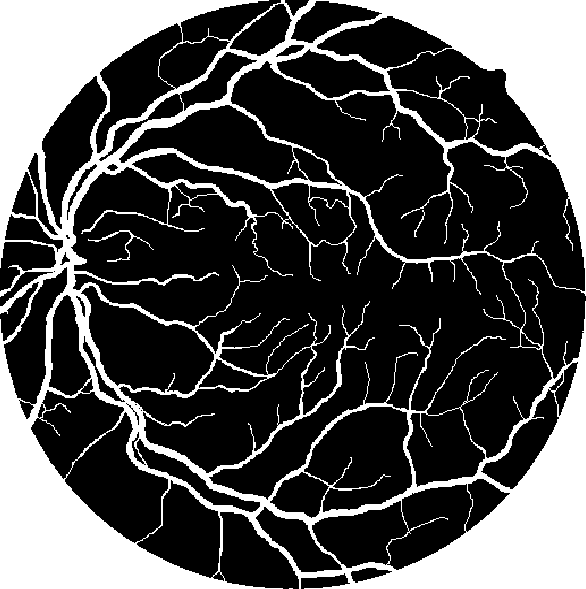}}
	\subfigure[Expert.
	\label{fig:drive_gt2}]{\includegraphics[width=0.24\textwidth]
		{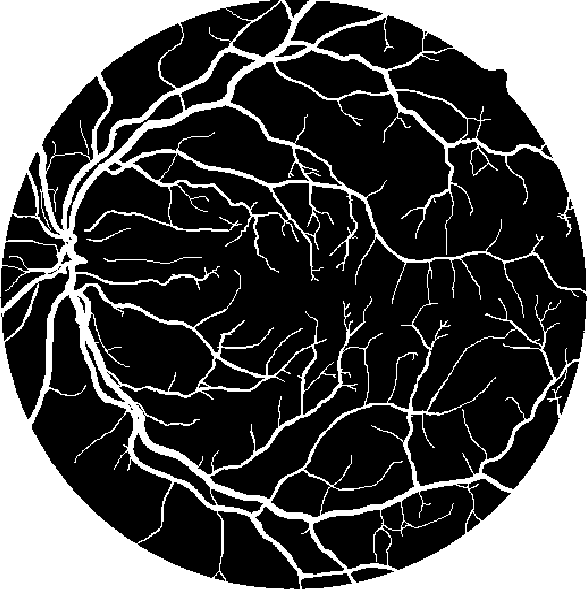}}
	\subfigure[Hessian.
	\label{fig:drive_jerman}]{\includegraphics[width=0.24\textwidth]
		{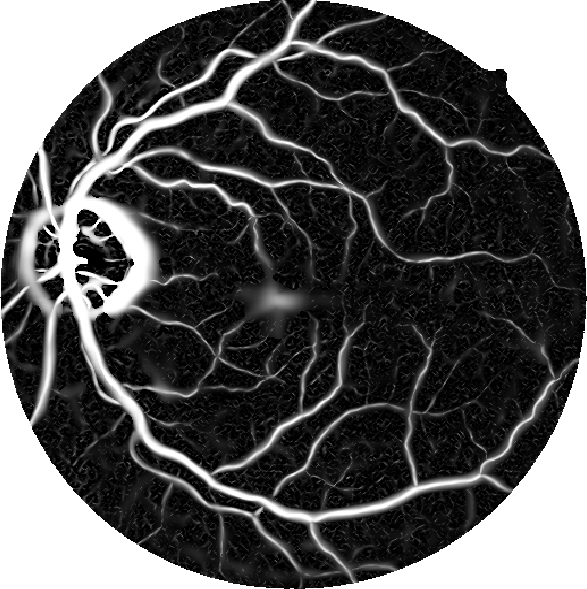}}	
	\subfigure[CSC-$\ell_2$.
	\label{fig:drive_cscl2}]{\includegraphics[width=0.24\textwidth]
		{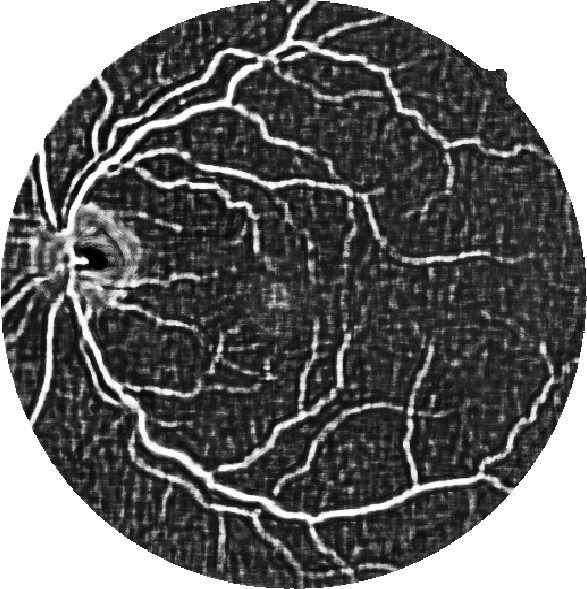}}
	\subfigure[CSC-$\ell_1$.
	\label{fig:drive_cscl1}]{\includegraphics[width=0.24\textwidth]
		{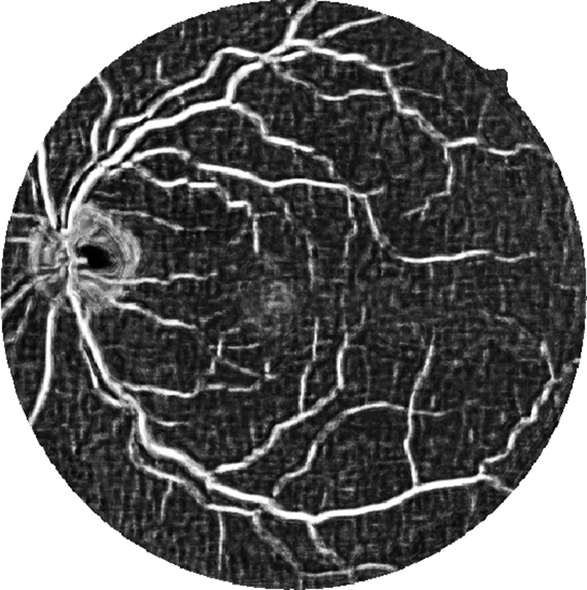}}
	\subfigure[$\alpha$CSC.
	\label{fig:drive_alphacsc}]{\includegraphics[width=0.24\textwidth]
		{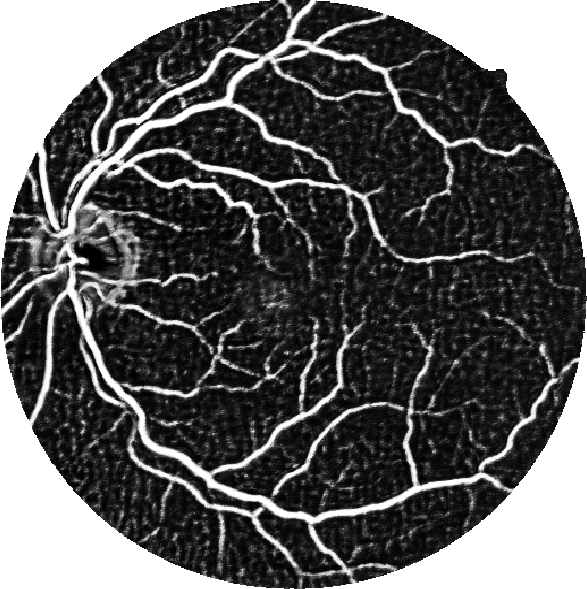}}			
	\subfigure[GCSC.
	\label{fig:drive_gcsc}]{\includegraphics[width=0.24\textwidth]
		{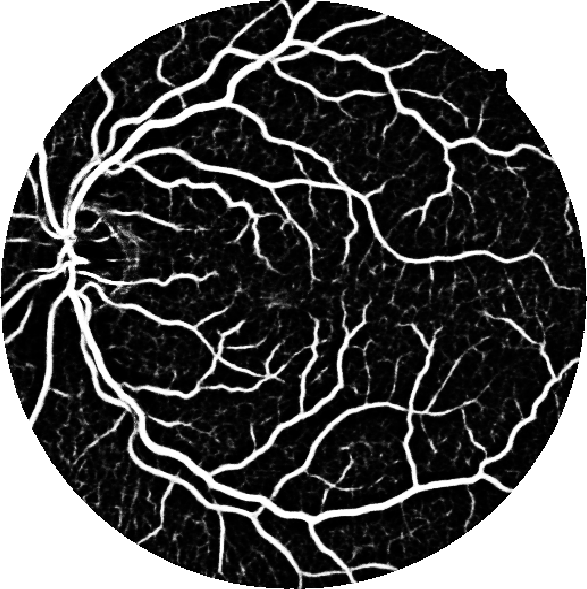}}
	\vspace{-10px}		
	\caption{Vessel segmentation results on a test retinal image from \textit{DRIVE}.}
	\label{fig:vessel_drive}
\end{figure*}

\begin{figure*}[htb]
	\centering		
	\subfigure[Image.
	\label{fig:stare_ori_img}]{\includegraphics[width=0.24\textwidth]
		{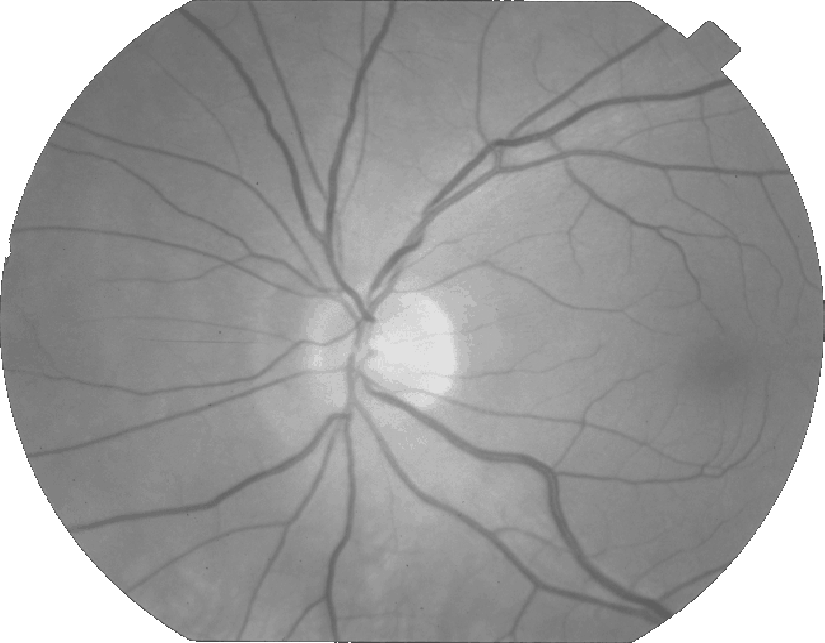}}
	\subfigure[Ground truth.
	\label{fig:stare_gt}]{\includegraphics[width=0.24\textwidth]
		{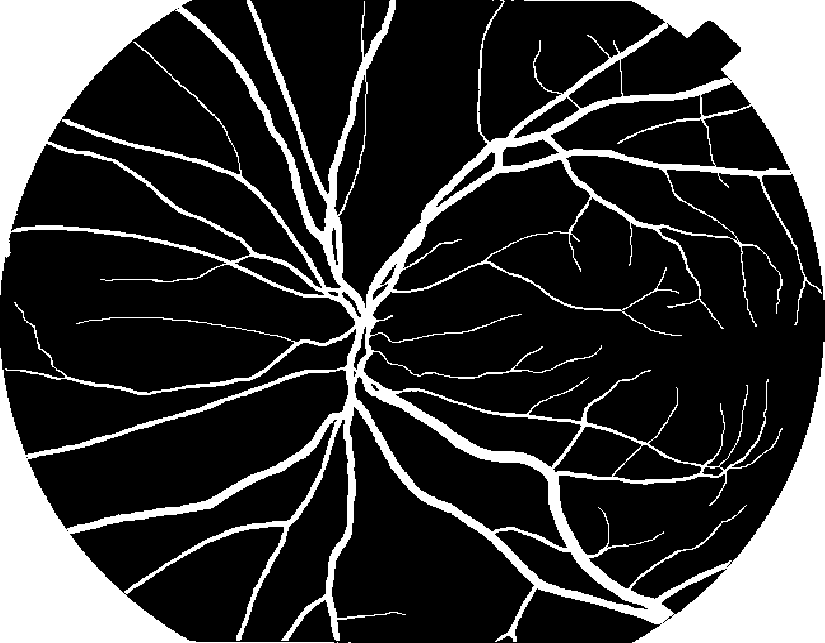}}
	\subfigure[Expert.
	\label{fig:stare_gt2}]{\includegraphics[width=0.24\textwidth]
		{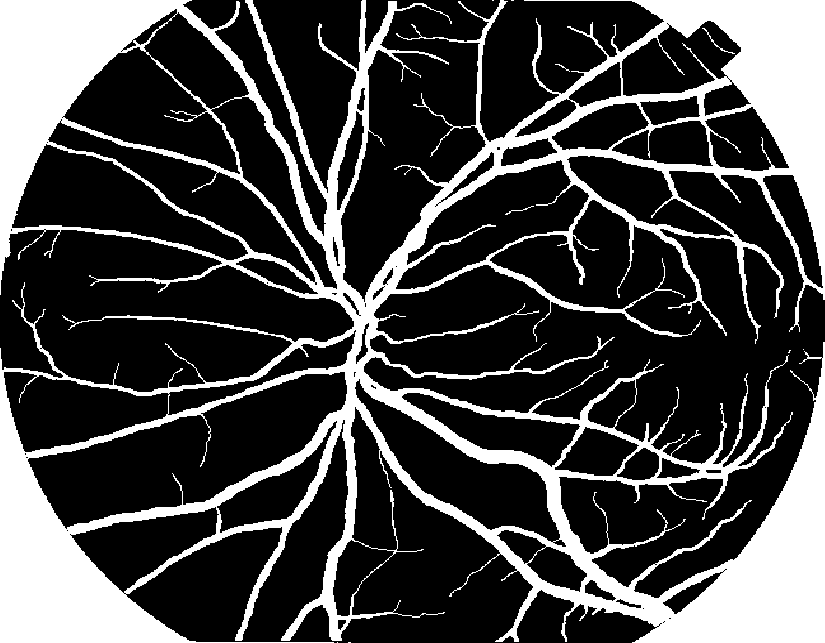}}
	\subfigure[Hessian.
	\label{fig:stare_jerman}]{\includegraphics[width=0.24\textwidth]
		{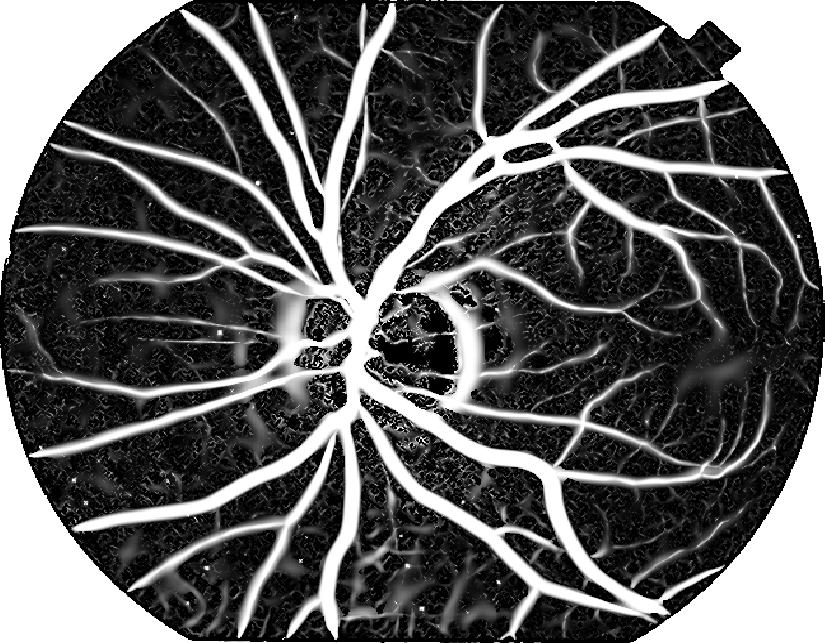}}	
	\subfigure[CSC-$\ell_2$.
	\label{fig:stare_cscl2}]{\includegraphics[width=0.24\textwidth]
		{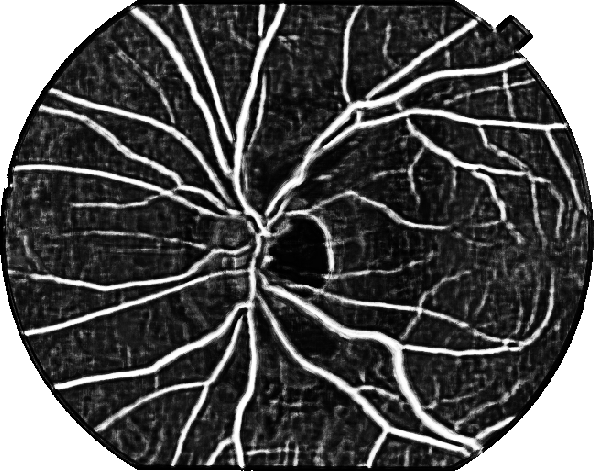}}
	\subfigure[CSC-$\ell_1$.
	\label{fig:stare_cscl1}]{\includegraphics[width=0.24\textwidth]
		{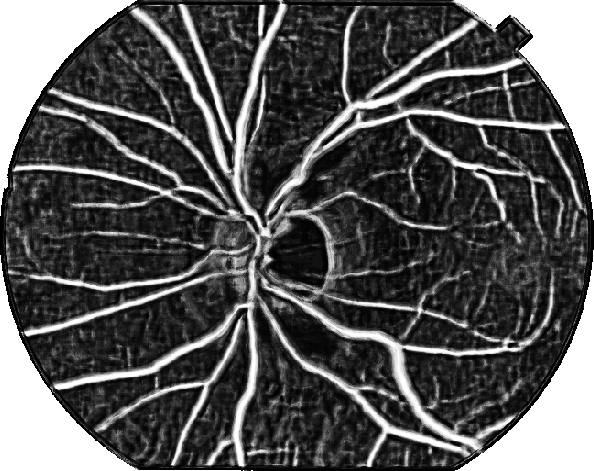}}
	\subfigure[$\alpha$CSC.
	\label{fig:stare_alphacsc}]{\includegraphics[width=0.24\textwidth]
		{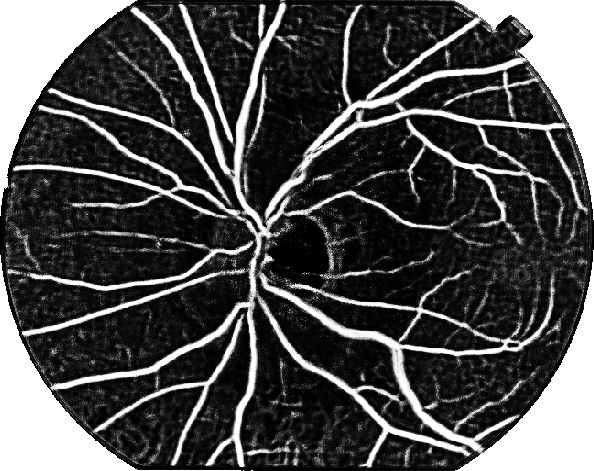}}			
	\subfigure[GCSC.
	\label{fig:stare_gcsc}]{\includegraphics[width=0.24\textwidth]
		{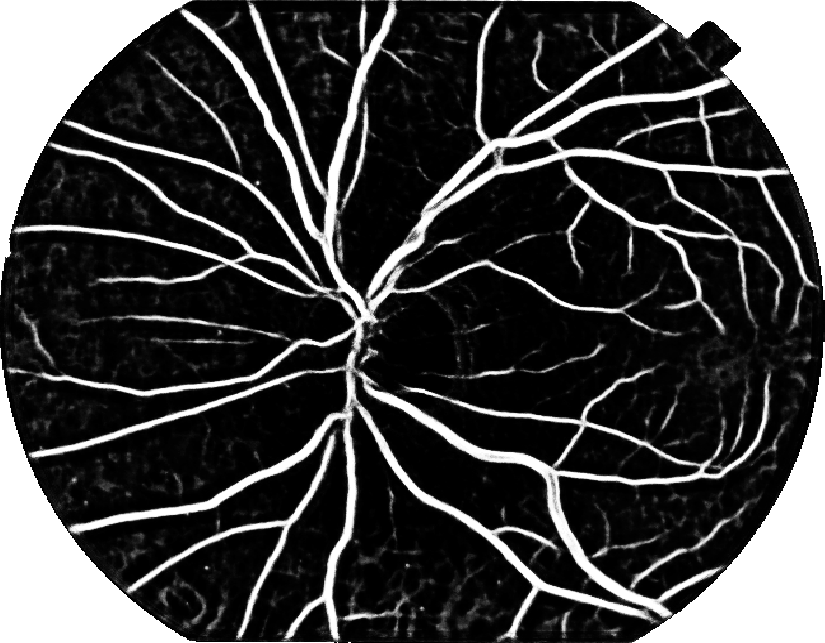}}		
	\vspace{-10px}
	\caption{Vessel segmentation results on a test retinal image from \textit{STARE}.}
	\label{fig:vessel_stare}
\end{figure*}

\section{Conclusion}

In this paper, we propose a CSC method which is able to deal with various kinds of noises.
We model the noises by Gaussian mixture model, and solve it by expectation-maximization algorithm. 
In the maximization step, the problem reduces to be a weighted CSC problem and we use a nonconvex and inexact accelerated proximal gradient algorithm without alternating.
Extensive experiments on synthetic and real noisy biomedical data sets
show that our method can model the complicated noises well and in turn obtain high-quality filters and representation.

\bibliographystyle{IEEEtran}
\bibliography{paper}

\begin{thebibliography}{10}
\providecommand{\url}[1]{#1}
\csname url@samestyle\endcsname
\providecommand{\newblock}{\relax}
\providecommand{\bibinfo}[2]{#2}
\providecommand{\BIBentrySTDinterwordspacing}{\spaceskip=0pt\relax}
\providecommand{\BIBentryALTinterwordstretchfactor}{4}
\providecommand{\BIBentryALTinterwordspacing}{\spaceskip=\fontdimen2\font plus
\BIBentryALTinterwordstretchfactor\fontdimen3\font minus
  \fontdimen4\font\relax}
\providecommand{\BIBforeignlanguage}[2]{{%
\expandafter\ifx\csname l@#1\endcsname\relax
\typeout{** WARNING: IEEEtran.bst: No hyphenation pattern has been}%
\typeout{** loaded for the language `#1'. Using the pattern for}%
\typeout{** the default language instead.}%
\else
\language=\csname l@#1\endcsname
\fi
#2}}
\providecommand{\BIBdecl}{\relax}
\BIBdecl

\bibitem{aharon2006rm}
M.~Aharon, M.~Elad, and A.~Bruckstein, ``{K-SVD}: An algorithm for designing
  overcomplete dictionaries for sparse representation,'' \emph{IEEE
  Transactions on Signal Processing}, vol.~54, no.~11, pp. 4311--4322, 2006.

\bibitem{lee2007efficient}
H.~Lee, A.~Battle, R.~Raina, and A.~Ng, ``Efficient sparse coding algorithms,''
  in \emph{Advances in Neural Information Processing Systems}, 2007, pp.
  801--808.

\bibitem{mairal2009non}
J.~Mairal, F.~Bach, J.~Ponce, G.~Sapiro, and A.~Zisserman, ``Non-local sparse
  models for image restoration,'' in \emph{International Conference on Computer
  Vision}, 2009, pp. 2272--2279.

\bibitem{yang2009linear}
J.~Yang, K.~Yu, Y.~Gong, and T.~Huang, ``Linear spatial pyramid matching using
  sparse coding for image classification,'' in \emph{Conference on Computer
  Vision and Pattern Recognition}, 2009, pp. 1794--1801.

\bibitem{zeiler2010deconvolutional}
M.~Zeiler, D.~Krishnan, G.~Taylor, and R.~Fergus, ``Deconvolutional networks,''
  in \emph{IEEE Conference on Computer Vision and Pattern Recognition}, 2010,
  pp. 2528--2535.

\bibitem{zhu2015convolutional}
Y.~Zhu and S.~Lucey, ``Convolutional sparse coding for trajectory
  reconstruction,'' \emph{IEEE Transactions on Pattern Analysis and Machine
  Intelligence}, vol.~37, no.~3, pp. 529--540, 2015.

\bibitem{heide2015fast}
F.~Heide, W.~Heidrich, and G.~Wetzstein, ``Fast and flexible convolutional
  sparse coding,'' in \emph{IEEE Conference on Computer Vision and Pattern
  Recognition}, 2015, pp. 5135--5143.

\bibitem{cogliati2016context}
A.~Cogliati, Z.~Duan, and B.~Wohlberg, ``Context-dependent piano music
  transcription with convolutional sparse coding,'' \emph{IEEE/ACM Transactions
  on Audio Speech and Language Processing}, vol.~24, no.~12, pp. 2218--2230,
  2016.

\bibitem{choudhury2017consensus}
B.~Choudhury, R.~Swanson, F.~Heide, G.~Wetzstein, and W.~Heidrich, ``Consensus
  convolutional sparse coding,'' in \emph{IEEE International Conference on
  Computer Vision}.\hskip 1em plus 0.5em minus 0.4em\relax IEEE, 2017, pp.
  4290--4298.

\bibitem{wang2018scsc}
Y.~Wang, Q.~Yao, J.~T. Kwok, and L.~M. Ni, ``Online convolutional sparse coding
  with sample-dependent dictionary,'' in \emph{International Conference on
  Machine Learning}, 2018, pp. 5209--5218.

\bibitem{andilla2014sparse}
F.~D. Andilla and F.~A. Hamprecht, ``Sparse space-time deconvolution for
  calcium image analysis,'' in \emph{Advances in Neural Information Processing
  Systems}, 2014, pp. 64--72.

\bibitem{sironi2015learning}
A.~Sironi, B.~Tekin, R.~Rigamonti, V.~Lepetit, and P.~Fua, ``Learning separable
  filters,'' \emph{IEEE Transactions on Pattern Analysis and Machine
  Intelligence}, vol.~37, no.~1, pp. 94--106, 2015.

\bibitem{chang2017unsupervised}
H.~Chang, J.~Han, C.~Zhong, A.~Snijders, and J.~Mao, ``Unsupervised transfer
  learning via multi-scale convolutional sparse coding for biomedical
  applications,'' \emph{IEEE Transactions on Pattern Analysis and Machine
  Intelligence}, 2017.

\bibitem{jas2017learning}
M.~Jas, T.~La~Tour, U.~Simsekli, and A.~Gramfort, ``Learning the morphology of
  brain signals using alpha-stable convolutional sparse coding,'' in
  \emph{Advances in Neural Information Processing Systems}, 2017, pp.
  1099--1108.

\bibitem{gu2015convolutional}
S.~Gu, W.~Zuo, Q.~Xie, D.~Meng, X.~Feng, and L.~Zhang, ``Convolutional sparse
  coding for image super-resolution,'' in \emph{International Conference on
  Computer Vision}, 2015, pp. 1823--1831.

\bibitem{peter2017sparse}
S.~Peter, E.~Kirschbaum, M.~Both, L.~Campbell, B.~Harvey, C.~Heins,
  D.~Durstewitz, F.~Diego, and F.~A. Hamprecht, ``Sparse convolutional coding
  for neuronal assembly detection,'' in \emph{Advances in Neural Information
  Processing Systems}, 2017, pp. 3678--3688.

\bibitem{hitziger2017adaptive}
S.~Hitziger, M.~Clerc, S.~Saillet, C.~B{\'e}nar, and T.~Papadopoulo, ``Adaptive
  waveform learning: a framework for modeling variability in neurophysiological
  signals,'' \emph{IEEE Transactions on Signal Processing}, vol.~65, no.~16,
  pp. 4324--4338, 2017.

\bibitem{tseng2001convergence}
P.~Tseng, ``Convergence of a block coordinate descent method for
  nondifferentiable minimization,'' \emph{Journal of optimization theory and
  applications}, vol. 109, no.~3, pp. 475--494, 2001.

\bibitem{kavukcuoglu2010learning}
K.~Kavukcuoglu, P.~Sermanet, Y.~Boureau, K.~Gregor, M.~Mathieu, and Y.~LeCun,
  ``Learning convolutional feature hierarchies for visual recognition,'' in
  \emph{Advances in Neural Information Processing Systems}, 2010, pp.
  1090--1098.

\bibitem{bristow2013fast}
H.~Bristow, A.~Eriksson, and S.~Lucey, ``Fast convolutional sparse coding,'' in
  \emph{IEEE Conference on Computer Vision and Pattern Recognition}, 2013, pp.
  391--398.

\bibitem{wohlberg2016efficient}
B.~Wohlberg, ``Efficient algorithms for convolutional sparse representations,''
  \emph{IEEE Transactions on Image Processing}, vol.~25, no.~1, pp. 301--315,
  Jan. 2016.

\bibitem{sorel2016fast}
M.~{\v{S}}orel and F.~{\v{S}}roubek, ``Fast convolutional sparse coding using
  matrix inversion lemma,'' \emph{Digital Signal Processing}, vol.~55, pp.
  44--51, Aug. 2016.

\bibitem{papyan2017convolutional}
V.~Papyan, Y.~Romano, J.~Sulam, and M.~Elad, ``Convolutional dictionary
  learning via local processing,'' in \emph{International Conference on
  Computer Vision}, 2017, pp. 5296--5304.

\bibitem{boyd2011distributed}
S.~Boyd, N.~Parikh, E.~Chu, B.~Peleato, and J.~Eckstein, ``Distributed
  optimization and statistical learning via the alternating direction method of
  multipliers,'' \emph{Foundations and Trends in Machine Learning}, vol.~3,
  no.~1, pp. 1--122, 2011.

\bibitem{mandelbrot1960pareto}
B.~Mandelbrot, ``The pareto-levy law and the distribution of income,''
  \emph{International Economic Review}, vol.~1, no.~2, pp. 79--106, 1960.

\bibitem{gilks1995markov}
W.~R. Gilks, S.~Richardson, and D.~Spiegelhalter, \emph{Markov chain Monte
  Carlo in practice}.\hskip 1em plus 0.5em minus 0.4em\relax CRC press, 1995.

\bibitem{mallat1999wavelet}
S.~Mallat, \emph{A Wavelet Tour of Signal Processing}.\hskip 1em plus 0.5em
  minus 0.4em\relax Academic Press, 1999.

\bibitem{parikh2014proximal}
N.~Parikh and S.~Boyd, ``Proximal algorithms,'' \emph{Foundations and Trends in
  Optimization}, vol.~1, no.~3, pp. 127--239, 2014.

\bibitem{yao2017efficient}
Q.~Yao, J.~T. Kwok, F.~Gao, W.~Chen, and T.-Y. Liu, ``Efficient inexact
  proximal gradient algorithm for nonconvex problems,'' in \emph{International
  Joint Conferences on Artifical Intelligence}, 2017, pp. 3308--3314.

\bibitem{maz1996approximate}
V.~Maz'ya and G.~Schmidt, ``On approximate approximations using gaussian
  kernels,'' \emph{IMA Journal of Numerical Analysis}, vol.~16, no.~1, pp.
  13--29, 1996.

\bibitem{dempster1977maximum}
A.~P. Dempster, N.~M. Laird, and D.~B. Rubin, ``Maximum likelihood from
  incomplete data via the em algorithm,'' \emph{Journal of the Royal
  Statistical Society. Series B (Methodological)}, pp. 1--38, 1977.

\bibitem{cooley1969fast}
J.~W. Cooley, P.~A. Lewis, and P.~D. Welch, ``The fast fourier transform and
  its applications,'' \emph{IEEE Transactions on Education}, vol.~12, no.~1,
  pp. 27--34, 1969.

\bibitem{grippo2002nonmonotone}
L.~Grippo and M.~Sciandrone, ``Nonmonotone globalization techniques for the
  {B}arzilai-{B}orwein gradient method,'' \emph{Computational Optimization and
  Applications}, vol.~23, no.~2, pp. 143--169, 2002.

\bibitem{efron2004least}
B.~Efron, T.~Hastie, I.~Johnstone, R.~Tibshirani \emph{et~al.}, ``Least angle
  regression,'' \emph{The Annals of statistics}, vol.~32, no.~2, pp. 407--499,
  2004.

\bibitem{degraux2017online}
K.~Degraux, U.~S. Kamilov, P.~T. Boufounos, and D.~Liu, ``Online convolutional
  dictionary learning for multimodal imaging,'' in \emph{IEEE International
  Conference on Image Processing}, 2017, pp. 1617--1621.

\bibitem{liu2017online}
J.~Liu, C.~Garcia-Cardona, B.~Wohlberg, and W.~Yin, ``Online convolutional
  dictionary learning,'' in \emph{IEEE International Conference on Image
  Processing}, 2017, pp. 1707--1711.

\bibitem{wang2018ocsc}
Y.~Wang, Q.~Yao, J.~T. Kwok, and L.~M. Ni, ``Scalable online convolutional
  sparse coding,'' \emph{IEEE Transactions on Image Processing}, vol.~27,
  no.~10, pp. 4850 -- 4859, 2018.

\bibitem{byrd1995limited}
R.~H. Byrd, P.~Lu, J.~Nocedal, and C.~Zhu, ``A limited memory algorithm for
  bound constrained optimization,'' \emph{SIAM Journal on Scientific
  Computing}, vol.~16, no.~5, pp. 1190--1208, 1995.

\bibitem{meng2013robust}
D.~Meng and F.~De~La~Torre, ``Robust matrix factorization with unknown noise,''
  in \emph{IEEE International Conference on Computer Vision}.\hskip 1em plus
  0.5em minus 0.4em\relax IEEE, 2013, pp. 1337--1344.

\bibitem{chen2016robust}
X.~Chen, Z.~Han, Y.~Wang, Q.~Zhao, D.~Meng, and Y.~Tang, ``Robust tensor
  factorization with unknown noise,'' in \emph{IEEE Conference on Computer
  Vision and Pattern Recognition}.\hskip 1em plus 0.5em minus 0.4em\relax IEEE,
  2016, pp. 5213--5221.

\bibitem{dallerac2017updating}
G.~Dall{\'e}rac, M.~Graupner, J.~Knippenberg, R.~C.~R. Martinez, T.~F. Tavares,
  L.~Tallot, N.~El~Massioui, A.~Verschueren, S.~H{\"o}hn, J.~B. Bertolus
  \emph{et~al.}, ``Updating temporal expectancy of an aversive event engages
  striatal plasticity under amygdala control,'' \emph{Nature communications},
  vol.~8, p. 13920, 2017.

\bibitem{staal2004ridge}
J.~Staal, M.~D. Abr{\`a}moff, M.~Niemeijer, M.~A. Viergever, and
  B.~Van~Ginneken, ``Ridge-based vessel segmentation in color images of the
  retina,'' \emph{IEEE Transactions on Medical Imaging}, vol.~23, no.~4, pp.
  501--509, 2004.

\bibitem{hoover2000locating}
A.~Hoover, V.~Kouznetsova, and M.~Goldbaum, ``Locating blood vessels in retinal
  images by piecewise threshold probing of a matched filter response,''
  \emph{IEEE Transactions on Medical imaging}, vol.~19, no.~3, pp. 203--210,
  2000.

\bibitem{becker2013supervised}
C.~Becker, R.~Rigamonti, and P.~Lepetit, V.and~Fua, ``Supervised feature
  learning for curvilinear structure segmentation,'' in \emph{International
  Conference on Medical Image Computing and Computer-Assisted
  Intervention}.\hskip 1em plus 0.5em minus 0.4em\relax Springer, 2013, pp.
  526--533.

\bibitem{friedman2001greedy}
J.~H. Friedman, ``Greedy function approximation: a gradient boosting machine,''
  \emph{Annals of Statistics}, pp. 1189--1232, 2001.

\bibitem{jerman2016enhancement}
T.~Jerman, F.~Pernu{\v{s}}, B.~Likar, and {\v{Z}}.~{\v{S}}piclin, ``Enhancement
  of vascular structures in 3d and 2d angiographic images,'' \emph{IEEE
  Transactions on Medical Imaging}, vol.~35, no.~9, pp. 2107--2118, 2016.

\end{thebibliography}

\cleardoublepage
\newpage

\appendices

\section{CSC-$\ell_1$: Detailed Algorithm}
\label{app:csc_l1}

In robust learning, the $\ell_1$ loss,
which corresponds to the Laplace distribution, 
is often used to handle outliers.
Here, we present the detailed algorithm for CSC with $\ell_1$ loss, which is called CSC-$\ell_1$ in Section~\ref{sec:expts}.

The objective of CSC-$\ell_1$ is: 
\begin{align}
\label{eq:l1}
\min_{\{d_k\},\{z_{ik}\}}&
\sum_{i=1}^{N}
\left( 
\frac{1}{2}
\NML{{x}_i
\!-\!
	\sumDZ}{1}
\!+\!
\sum_{k=1}^K
\beta\NM{z_{ik}}{1}\right) 
\\\notag
\text{s.t.}&
\|{d}_k\|_2^2 \le 1, k =1,\dots,K. 
\end{align}
$d_k$'s and $z_{ik}$'s 
are updated by BCD  
until convergence.


\subsection{Dictionary Update}

Given $z_{ik}$'s, 
dictionary $d_k$'s are obtained by reformulating (\ref{eq:l1}) as:
\begin{eqnarray*}
\notag
&\min_{\{e_{i}\},\{d_k,v_k\}}&
\sum_{i=1}^{N}
\| e_{i}\|_1
\\\notag
&\text{s.t.}&
\|{v}_k\|_2^2 \le 1, \forall k, \\\notag
&&
e_{i}={x}_i-\sumDZ,  \forall i,\\\notag
&&
d_k = v_k , \forall k,
\end{eqnarray*}
where  
$e_{i}$'s and $v_k$'s are auxiliary variables.
This can then be solved by ADMM. 
We first form the augmented Lagrangian as
\begin{align}
	\notag
	\min_{\{e_{i},\alpha_{i}\},\{d_k,v_k,\theta_k\}}&
	\sum_{i=1}^{N}
	\left(
	\| e_{i}\|_1
	+ \alpha_{i}^\top
	\left(e_{i}-{x}_i+\sumDZ\right) +\frac{\rho}{2}\NML{e_{i}-{x}_i+\sumDZ}{2}^2\right)\\\notag
	&+ 
	\sum_{k=1}^{K}
	\left( 
	\theta_k^\top
	(d_k - v_k)
	+ \frac{\rho}{2}
	\NM{d_k - v_k}{2}^2	\right) 	
	\\\notag
	\text{s.t.}&
	\|{v}_k\|_2^2 \le 1, \forall k,
\end{align}
where 
$\rho$ is the ADMM penalty parameter, $\theta_k$'s and $\alpha_{i}$'s are dual variables.
At $\tau$th iteration, ADMM alternately updates 
$d^\tau_k$'s, $e_i^\tau$'s, $v^\tau_k$'s, $\alpha^{\tau}_{i}$'s and $\theta^{\tau}_k$'s until convergence.

$d^{\tau}_k$'s are updated as
\begin{align}\notag
\lefteqn{
\arg\min_{\{d_k\}}
\sum_{i=1}^{N}
\left( 
(\alpha_{i}^{\tau-1})^\top
\left(e_{i}^{\tau-1}-{x}_i+\sumDZ\right)+\frac{\rho}{2}\NML{e_{i}^{\tau-1}\!-\!{x}_i\!+\!\sumDZ}{2}^2\right)} \\\notag
&+ 
\sum_{k=1}^{K}
\left( 
(\theta_k^{\tau-1})^\top
(d_k \!-\! v_k^{\tau-1})
+\!\! \frac{\rho}{2}
\NM{d_k \!-\! v_k^{\tau-1}}{2}^2\right) 		\\
&=\arg\min_{\{d_k\}}
\sum_{i=1}^{N}
\NML{
	{x}_i\!-\!\sumDZ\!-\!e^{\tau-1}_{i}\!-\!\frac{\alpha^{\tau-1}_{i}}{\rho} }{2}^2
+
\sum_{k=1}^{K}
\NML{
	d_k-v^{\tau-1}_k+\frac{\theta^{\tau-1}_k}{\rho} }{2}^2.
\label{eq:cscl1_filter_d}
\end{align}
Convolution is more efficient in frequency domain, we update $d_k$'s therein. 
Let $\eta^{\tau-1}_{i} = x_i-e^{\tau-1}_{i}-\frac{\alpha^{\tau-1}_{i}}{\rho}$, we transform all variables to frequency domain (denoted as symbol with hat) using FFT as $\hat{d}_k =\FFT(C^\top d_k)$, $\hat{v}_k^{\tau-1} = \FFT(C^\top v_k^{\tau-1})$, $\hat{\theta}_k^{\tau-1} = \FFT(C^\top \theta_k^{\tau-1})$, $\hat{z}_{ik} = \FFT(z_{ik})$,  
$\hat{\eta}_{i}^{\tau-1} = \FFT(\eta_{i}^{\tau-1})$
and $\hat{\alpha}_{i}^{\tau-1} = \FFT(\alpha_{i}^{\tau-1})$ where $C$ is the padding matrix used as in \eqref{eq:csc_dic_fre}.
Using these frequency-domain variables, \eqref{eq:cscl1_filter_d} can be written as: 
\begin{align*}
	\min_{\{\hat{d}_k\}}
	\sum_{i=1}^{N}
	\NML{
		\hat{\eta}^{\tau-1}_{i}\!-\!\sumDZfre}{2}^2
	\!\!+\!\!
	\sum_{k=1}^{K}
	\NML{
		\hat{d}_k\!-\!\hat{v}^{\tau-1}_k\!+\!\frac{\hat{\theta}^{\tau-1}_k}{\rho} }{2}^2\!\!.	
\end{align*} 
This can be solved by closed form solution.
Using the reordering trick used in \cite{wohlberg2016efficient,wang2018ocsc}, 
we first put all $d_k$'s in the columns of $\hat{D}=[\hat{d}_1,\dots,\hat{d}_K]$, as well as 
$\hat{V}^{\tau-1}=[\hat{v}^{\tau-1}_1,\dots,\hat{v}^{\tau-1}_K]$, $\hat{\Theta}^{\tau-1}=[\hat{\theta}^{\tau-1}_1,\dots,\hat{\theta}^{\tau-1}_K]$, and 
$\hat{Z}_{i}=[\hat{z}_{i1},\dots,\hat{z}_{iK}]$.
Then $p$th row $\hat{D}^\tau(p,:)$ is updated as
\begin{align*}
\notag
\hat{D}^\tau(p,:)\!=\!
\left(
\sum_{i=1}^{N}{\eta}^{\tau-1}_{i}(p)\hat{Z}^{\star}_{i}(p,:) \!+\! \hat{V}^{\tau-1}(p,:) \!-\!\frac{\hat{\Theta}^{\tau-1}(p,:)}{\rho}\!\right)\cdot
\left(
\sum_{i=1}^{N}\hat{Z}^{\top}_i(:,p)\hat{Z}^\star_i(p,:)\!+\!I\!\right )^{-1}.
\end{align*}
Then $d_k^\tau$ can be recovered as $C\iFFT(\hat{d}_k^\tau)$.

Each $e^{\tau}_{i}$ is then independently updated as
\begin{align}\notag
	\arg&\min_{e_i}
	\|  e_{i}\|_1\!+\!
	(\alpha_{i}^{\tau-1})^\top
	\left(e_{i}-{x}_i+\sum_{k=1}^{K}d^\tau_k*z_{ik}\right)
+\!\frac{\rho}{2}\NML{e_{i}\!-\!{x}_i\!+\!\sum_{k=1}^{K}d^\tau_k*z_{ik}}{2}^2
\\\notag
&=	\arg\min_{e_i}
\|  e_{i}\|_1\!+\!
\frac{\rho}{2}\NML{
	{x}_i\!-\!\sum_{k=1}^{K}d^\tau_k*z_{ik}\!-\!e_{i}\!-\!\frac{\alpha^{\tau-1}_{i}}{\rho} }{2}^2
\end{align}
with closed-form solution for $p$th element:
\begin{equation}
e^{\tau}_{i}(p) =\text{sign}(\nu(p))
\odot\max(|\nu(p)| - \frac{1}{\rho}, 0),
\label{eq:cscl1_e}
\end{equation}
where $\nu = {x}_i-\sum_{k=1}^{K}d^{\tau}_k*z_{ik} - \frac{\alpha^{\tau-1}_{i}}{\rho}$.

$v^{\tau}_k$ is updated as
\begin{align*}
\lefteqn{
	\arg\min_{v_k}(\theta_k^{\tau-1})^\top
(d_k^\tau \!-\! v_k)+\!\! \frac{\rho}{2}
\NM{d_k^\tau \!-\! v_k}{2}^2}\\
&=\arg\min_{v_k}
\frac{\rho}{2}\NML{d_k^\tau-v_k+\frac{\theta^{\tau-1}_k}{\rho} }{2}^2.
\end{align*}
with closed-form solution 
\[ v^{\tau}_k=\frac{d^{\tau}_k+\frac{\theta^{\tau-1}_k}{\rho}}{\max(\|d^{\tau}_k+\frac{\theta^{\tau-1}_k}{\rho}\|_2,1)}. \]

Finally,
$\alpha^{\tau}_{i}$ and
$\theta^{\tau}_k$ 
are updated as:
\begin{eqnarray*}
\alpha^{\tau}_{i} & = & \alpha^{\tau-1}_{i} + \rho \left(e^{\tau}_{i} - {x}_i + \sum_{k=1}^{K}d^{\tau}_k*z_{ik}\right), \\
\theta^{\tau}_k & =& \theta^{\tau-1}_{k}+\rho(d^{\tau}_k-v^{\tau}_k). 
\end{eqnarray*}


\subsection{Code Update}
Given ${d}_k$'s, 
the codes ${z}_{ik}$'s for for each sample $i$
can be obtained one by one 
by 
rewriting (\ref{eq:l1}) as:
\begin{eqnarray}\notag
&\min_{e_{i},\{z_{ik},u_{ik}\}}&
\| e_{i}\|_1
+ 
\sum_{k=1}^K
\beta\|u_{ik}\|_1
\\\nonumber
&\text{s.t.}&
e_{i}={x}_i-\sumDZ,\\\nonumber
&&
u_{ik} = z_{ik},\forall k, 
\end{eqnarray}
where $e_{i}$ and $u_{ik}$'s are auxiliary variables.

Using ADMM, we introduce $\alpha_{i}$ and $\lambda_{ik}'s$ as dual variables,
then the augmented Lagrangian is constructed as 
\begin{align}
\notag
\!\min_{e_{i}\!,\alpha_{i}\!,\{\!z_{ik},\!u_{ik},\!\lambda_{ik}\!\}}\!&
\| e_{i}\|_1
\!+\! \alpha_{i}^\top
\left(e_{i}\!-\!{x}_i\!+\!\sumDZ\right)
\!+\!\frac{\rho}{2}\NML{e_{i}\!-\!{x}_i\!+\!\sumDZ}{2}^2
\\\notag
&\!+\!
\sum_{k=1}^{K}\!\!
\left( 
\beta\NM{u_{ik}}{1}
\!+\! 
\lambda_{ik}^\top
(z_{ik} \!-\! u_{ik})
\!+\! \frac{\rho}{2}
\NM{z_{ik} \!-\! u_{ik}}{2}^2\right)\! .
\end{align}
At $\tau$th iteration, ADMM alternately updates 
$z^\tau_{ik}$'s, $e_i^\tau$, $u^\tau_{ik}$'s, $\alpha^{\tau}_{i}$ and $\lambda^{\tau}_{ik}$'s until convergence.

$z^{\tau}_{ik}$'s are updated as
\begin{align}
\notag
\lefteqn{
\arg\min_{\{z_{ik}\}}
(\alpha_{i}^{\tau-1})^\top
\left(e_{i}^{\tau-1}\!-\!{x}_i\!+\!\sum_{k=1}^{K}d_k*z_{ik}\right)
\!+\!\frac{\rho}{2}\NML{e_{i}^{\tau-1}\!-\!{x}_i\!+\!\sum_{k=1}^{K}d_k*z_{ik}}{2}^2}\\\notag
&\!+\! \sum_{k=1}^{K}\left( 
(\lambda^{\tau-1}_{ik})^\top
(z_{ik} \!-\! u^{\tau-1}_{ik})
\!+\! \frac{\rho}{2}
\NM{z_{ik} \!-\! u^{\tau-1}_{ik}}{2}^2\right) \\\label{eq:cscl1_code_z}
&=\arg 
\min_{\{z_{ik}\}}
\NML{
	{x}_i-\sum_{k=1}^{K}d_k*z_{ik} -e^{\tau-1}_{i}-\frac{\alpha^{\tau-1}_{i}}{\rho} }{2}^2
+\sum_{k=1}^{K} 
\NML{z_{ik}-u^{\tau-1}_{ik}+\frac{\lambda^{\tau-1}_{ik}}{\rho}}{2}^2.
\end{align}
Similar to how we solve \eqref{eq:cscl1_filter_d}, we 
let $\eta^{\tau-1}_{i} = x_i-e^{\tau-1}_{i}-\frac{\alpha^{\tau-1}_{i}}{\rho}$, and transform all variables to frequency domain as
$\hat{d}_k =\FFT(C^\top d_k)$, 
$\hat{z}_{ik} = \FFT(z_{ik})$, $\hat{u}_{ik}^{\tau-1} = \FFT(u_{ik}^{\tau-1})$, $\hat{\lambda}_{ik}^{\tau-1} = \FFT(\lambda_{ik}^{\tau-1})$,  
$\hat{\eta}_{i}^{\tau-1} = \FFT(\eta_{i}^{\tau-1})$ and
$\hat{\alpha}_{i}^{\tau-1} = \FFT(\alpha_{i}^{\tau-1})$.
Then \eqref{eq:cscl1_code_z} can be written as: 
\begin{align*}
\min_{\{\hat{z}_{ik}\}}
\NML{
	\hat{\eta}^{\tau-1}_{i}
	\!-\!
	\sumDZfre}{2}^2
\!+\!
\sum_{k=1}^{K}
\NML{
	\hat{z}_{ik}\!-\!\hat{u}^{\tau-1}_{ik}\!+\!\frac{\hat{\lambda}^{\tau-1}_{ik}}{\rho} }{2}^2.
\end{align*}
Using the reordering trick, we define
$\hat{D}=[\hat{d}_1,\dots,\hat{d}_K]$, 
$\hat{Z}_{i}=[\hat{z}_{i1},\dots,\hat{z}_{iK}]$, 
$\hat{U}^{\tau-1}_{i}=[\hat{u}^{\tau-1}_{i1},\dots,\hat{u}^{\tau-1}_{iK}]$,
and 
$\hat{\Lambda}^{\tau-1}_{i}=[\hat{\Lambda}^{\tau-1}_{i1},\dots,\hat{\Lambda}^{\tau-1}_{iK}]$.
Then $p$th row $\hat{Z}^{\tau}_{i}(p,:)$ is updated as
\begin{align}\notag
\hat{Z}^{\tau}_{i}(p,:)=
\left (\hat{\eta}_{i}^{\tau-1}(p)\hat{D}^{\star}(p,:) \!+\! \hat{U}^{\tau-1}_{i}(p,:) \!-\!\frac{\hat{\Lambda}^{\tau-1}_{i}(p,:)}{\rho}\right )
\cdot
\left(\hat{D}^{\top}(:,p)\hat{D}^\star(p,:)+I\right )^{-1}.
\notag
\end{align}
Then $z^\tau_{ik}$ is recovered as $\iFFT(\hat{z}^\tau_{ik})$. 


Each $e^{\tau}_{i}$ is then independently updated as
\begin{align}\notag
\arg&\min_{e_i}
	\|  e_{i}\|_1\!+\!
	(\alpha_{i}^{\tau-1})^\top
	\left(e_{i}-{x}_i+\sum_{k=1}^{K}d_k*z^\tau_{ik}\right)
+\!\frac{\rho}{2}\NML{e_{i}\!-\!{x}_i\!+\!\sum_{k=1}^{K}d_k*z^\tau_{ik}}{2}^2
\\\notag
&=	\arg\min_{e_i}
\|  e_{i}\|_1\!+\!
\frac{\rho}{2}
\NML{
	{x}_i\!-\!\sum_{k=1}^{K}d_k*z^\tau_{ik}\!-\!e_{i}\!-\!\frac{\alpha^{\tau-1}_{i}}{\rho} }{2}^2.
\end{align}
Similar to \eqref{eq:cscl1_e}, $e^{\tau}_{i}(p)$ is updated in closed-form as
\begin{equation*}
e^{\tau}_{i}(p) =\text{sign}(\nu(p))
\odot\max(|\nu(p)| - \frac{1}{\rho}, 0),
\end{equation*}
where $\nu = {x}_i-\sum_{k=1}^{K}d_k*z^{\tau}_{ik} - \frac{\alpha^{\tau-1}_{i}}{\rho}$.

Each
$u^{\tau}_{ik}$ is updated as
\begin{align}\notag
\!\!\!
\lefteqn{\arg\min_{u_{ik}}
(\lambda^{\tau-1}_{ik})^\top
(z_{ik}^{\tau} \!-\! u_{ik})
\!+\! \frac{\rho}{2}
\NM{z_{ik}^\tau \!-\! u_{ik}}{2}^2
\!+\!\beta\|u_{ik}\|_1}\\\notag
&=\arg 
\min_{u_{ik}}
\beta\|u_{ik}\|_1\!+\!
\frac{\rho}{2}\NML{z_{ik}^\tau\!-\!u_{ik}\!+\!\frac{\lambda^{\tau-1}_{ik}}{\rho}}{2}^2.
\end{align}
with closed-form solution:
\begin{equation*}
u^{\tau}_{ik}(p) =\text{sign}(\psi(p))\odot\max(|\psi(p)| - \frac{\beta}{\rho}, 0)
\end{equation*}
where $\psi = z^{\tau}_{ik}+\frac{\lambda^{\tau-1}_{ik}}{\rho}$.

Finally,
$\alpha^{\tau}_{i}$ and
$\lambda^{\tau}_{ik}$ 
are updated as:
\begin{align*}
\alpha^{\tau}_{i}&=\alpha^{\tau-1}_{i}+\rho \left(e^{\tau}_{i} - {x}_{i} + \sum_{k=1}^{K}d_k*z^{\tau}_{ik}\right),\\
\lambda^{\tau}_{ik}&=\lambda^{\tau-1}_{ik}+ \rho (z^{\tau}_{ik}-u^{\tau}_{ik}).
\end{align*}

\section{Solving \eqref{eq:wcsc_F} by BCD}
\label{app:wcsc_bcd}
As stated in the main text, we mention that solving \eqref{eq:wcsc_F} by niAPG is faster than BCD with both $d_k$'s and $z_{ik}$'s being solved by ADMM.
Here we detail how to solve \eqref{eq:wcsc_F} by BCD. 

\subsection{Filter Update}
\label{app:cscl2_filter} 
Given $z_{ik}$'s, $d_k$'s are obtained by solving the following objective:
\begin{eqnarray*}
	&\min_{\{e_{gi}\},\{d_k,v_k\}}&
	\frac{1}{2}
	\sum_{i=1}^{N}
	\sum_{g=1}^{G}
	\NM{w_{gi}\odot e_{gi} }{2}^2
	\\\nonumber
	&\text{s.t.}&
	\|{v}_k\|_2^2 \le 1, \forall k, \\
	\nonumber
	&&
	e_{gi}=x_i-\sumDZ-\mu_g,  \forall i,g\\\nonumber
	&&
	v_k = d_k, \forall k.
\end{eqnarray*}
where  
$e_{gi}$'s and $v_k$'s are auxiliary variables.
This can then be solved by ADMM. 
We first form the augmented Lagrangian as
\begin{align}
\notag
\min_{\{e_{gi},\alpha_{gi}\},\{d_k,v_k,\theta_k\}}&
\frac{1}{2}
	\sum_{i=1}^{N}
\sum_{g=1}^{G}	\left( 
\NM{w_{gi}\odot e_{gi} }{2}^2\!+\! \alpha_{gi}^\top
\left(e_{gi}-{x}_i\!+\!\sumDZ\!+\!\mu_{gi}\right)
\!+\!\frac{\rho}{2}\NML{e_{gi}-{x}_i\!+\!\sumDZ\!+\!\mu_{gi}}{2}^2\right)
\\\notag
&
\!+\! 
\sum_{k=1}^{K}
\left( 
\theta_k^\top
(d_k - v_k)
\!+\! \frac{\rho}{2}
\NM{d_k - v_k}{2}^2	\right) 	
\\\notag
\text{s.t.}&
\|{v}_k\|_2^2 \le 1, \forall k,
\end{align}
where 
$\rho$ is the ADMM penalty parameter, $\theta_{k}$'s and $\alpha_{gi}$'s are dual variables.
At $\tau$th iteration, ADMM alternately updates 
$d^\tau_k$'s, $e_{gi}^\tau$'s, $v^\tau_k$'s, $\alpha^{\tau}_{gi}$'s and $\theta^{\tau}_k$'s until convergence.

$d^{\tau}_k$'s are updated as
\begin{align}\notag
\lefteqn{
	\arg\min_{\{d_k\}}
	\sum_{i=1}^{N} \sum_{g=1}^G	\left( 
	(\alpha_{gi}^{\tau-1})^\top
	\left(e_{gi}^{\tau-1}\!-\!{x}_i\!+\!\sumDZ\!+\!\mu_{gi}\right)
+\!\frac{\rho}{2}\NML{e_{gi}^{\tau-1}\!-\!{x}_i\!+\!\sumDZ\!+\!\mu_{gi}}{2}^2\right)} \\\notag
&\!+\! 
\sum_{k=1}^{K}
\left( 
(\theta_k^{\tau-1})^\top
(d_k \!-\! v_k^{\tau-1})
+\!\! \frac{\rho}{2}
\NM{d_k \!-\! v_k^{\tau-1}}{2}^2\right) 		\\\label{eq:wcsc_admm_d}
&=\arg\min_{\{d_k\}}
\sum_{i=1}^{N}
\sum_{g=1}^{G}
\NML{
	x_i\!-\!\sumDZ\!-\!\mu_g\!-\!e^{\tau-1}_{gi}\!-\!\frac{\alpha^{\tau-1}_{gi}}{\rho} }{2}^2
\!+\!
\sum_{k=1}^{K}
\NML{
	d_k\!-\!v^{\tau-1}_k\!+\!\frac{\theta^{\tau-1}_k}{\rho} }{2}^2.
\end{align} 
Since convolution is more efficient in frequency domain, we update $d_k$'s therein. 
Let $\eta^{\tau-1}_{gi} = x_i-\mu_g-e^{\tau-1}_{gi}-\frac{\alpha^{\tau-1}_{gi}}{\rho} $, and transform all variables to frequency domain (denoted as symbol with hat) using FFT as
 $\hat{d}_k =\FFT(C^\top d_k)$, $\hat{v}_k^{\tau-1} = \FFT(C^\top v_k^{\tau-1})$, $\hat{\theta}_k^{\tau-1} = \FFT(C^\top \theta_k^{\tau-1})$, 
 $\hat{z}_{ik} = \FFT(z_{ik})$,  
 $\hat{\eta}_{gi}^{\tau-1} = \FFT(\eta_{gi}^{\tau-1})$
 and $\hat{\alpha}_{gi}^{\tau-1} = \FFT(\alpha_{gi}^{\tau-1})$.
Using these frequency-domain variables, 
\eqref{eq:wcsc_admm_d} can be written as: 
\begin{align*}
	\min_{\{\hat{d}_k\}}
	\sum_{i=1}^{N}
	\sum_{g=1}^G
	\NML{
		\hat{\eta}^{\tau-1}_{gi}\!\!-\!\!\sumDZfre}{2}^2
	\!\!+\!
	\sum_{k=1}^{K}
	\NML{
		\hat{d}_k\!-\!\hat{v}^{\tau-1}_k\!+\!\frac{\theta^{\tau-1}_k}{\rho} }{2}^2.
\end{align*}
This can be solved by closed form solution. 
Using the reordering trick used in \cite{wohlberg2016efficient,wang2018ocsc}, 
we first put all $d_k$'s in the columns of $\hat{D}=[\hat{d}_1,\dots,\hat{d}_K]$, as well as 
$\hat{V}^{\tau-1}=[\hat{v}^{\tau-1}_1,\dots,\hat{v}^{\tau-1}_K]$, $\hat{\Theta}^{\tau-1}=[\hat{\theta}^{\tau-1}_1,\dots,\hat{\theta}^{\tau-1}_K]$, and 
$\hat{Z}_{i}=[\hat{z}_{i1},\dots,\hat{z}_{iK}]$.
Then $p$th row $\hat{D}^\tau(p,:)$ is updated as
\begin{align*}
\notag
\hat{D}^\tau(p,:)=
\left(\sum_{i=1}^{N}\zeta(p)\hat{Z}^{\star}_{i}(p,:) \!+\! \hat{V}^{\tau-1}(p,:) \!-\!\hat{\Theta}^{\tau-1}(p,:)\right)
\cdot
\left(G\sum_{i=1}^{N}\hat{Z}^{\top}_i(:,p)\hat{Z}^\star_i(p,:)\!+\!I\right )^{-1},
\label{eq:wcsc_d}
\end{align*}
where $\zeta=\sum_{g=1}^G\hat{\eta}^{\tau-1}_{gi}$.
$d_k^\tau$ can be recovered as $C\iFFT(\hat{d}_k^\tau)$.

Each $e^{\tau}_{gi}$ is independently updated as
\begin{align}\notag
	\arg&\min_{e_{gi}}
	\frac{1}{2}
	\NM{w_{gi}\odot e_{gi} }{2}^2
	\!+\!
	(\alpha_{gi}^{\tau-1})^\top
	\left(e_{gi}\!-\!{x}_i\!+\!\sum_{k=1}^{K}d^\tau_k*z_{ik}\!+\!\mu_g\right)
\!+\!\frac{\rho}{2}\NML{e_{gi}\!-\!{x}_i\!+\!\sum_{k=1}^{K}d^\tau_k*z_{ik}+\mu_g}{2}^2
\\\notag
&=	\arg\min_{e_{gi}}
\NM{w_{gi}\odot e_{gi} }{2}^2
\!+\!\rho 
\NML{
	{x}_i\!-\!\sum_{k=1}^{K}d^\tau_k*z_{ik}\!-\!\mu_g\!-\!e_{gi}\!-\!\frac{\alpha^{\tau-1}_{gi}}{\rho} }{2}^2,
\end{align}
with closed-form solution for $p$th element:
\begin{equation}
e^{\tau}_{gi}(p)= \frac{\rho \nu(p)}{w^2_{gi}(p)+\rho},
\label{eq:wcsc_admm_e}
\end{equation}
where $\nu = {x}_i-\sum_{k=1}^{K}d^{\tau}_k*z_{ik} -\mu_g- \frac{\alpha^{\tau-1}_{gi}}{\rho}$.

$v^{\tau}_k$ is updated as
\begin{align*}
\lefteqn{
	\arg\min_{v_k}(\theta_k^{\tau-1})^\top
	(d_k^\tau \!-\! v_k)+\!\! \frac{\rho}{2}
	\NM{d_k^\tau \!-\! v_k}{2}^2}\\
&=\arg\min_{v_k}
\frac{1}{2}
\NML{d_k^\tau-v_k+\frac{\theta^{\tau-1}_k}{\rho} }{2}^2.
\end{align*}
with closed-form solution 
\[ v^{\tau}_k=\frac{d^{\tau}_k+\frac{\theta^{\tau-1}_k}{\rho}}{\max(\|d^{\tau}_k+\frac{\theta^{\tau-1}_k}{\rho}\|_2,1)}. \]

Finally,
$\alpha^{\tau}_{gi}$ and
$\theta^{\tau}_k$ 
are updated as:
\begin{eqnarray*}
	\alpha^{\tau}_{gi} & = & \alpha^{\tau-1}_{gi} + \rho \left(e^{\tau}_{gi} - {x}_i + \sum_{k=1}^{K}d^{\tau}_k*z_{ik} +\mu_g\right), \\
	\theta^{\tau}_k & =& \theta^{\tau-1}_{k}+\rho(d^{\tau}_k-v^{\tau}_k). 
\end{eqnarray*}

\subsection{Code Update}
\label{app:cscl2_z}
In the codes update subproblem, $z_{ik}$'s can be independently updated for each $i$.
The objective to solve is: 
\begin{eqnarray*}
	\min_{\{z_{ik},u_{ik}\},\{e_{gi}\}} 
	&&
	\frac{1}{2}
	\sum_{g=1}^{G}
	\NM{w_{gi}\odot e_{gi} }{2}^2
	+\beta\sum_{k=1}^{K}\|u_{ik}\|_1
	,\\
	\text{s.t.}
	&&~e_{gi}=x_i-\sumDZ-\mu_g,  \forall g,\\
	&&~u_{ik}={z}_{ik}, \forall k.
\end{eqnarray*}
where $e_{gi}$'s and $u_{ik}$'s are auxiliary variables.

Using ADMM, we introduce $\alpha_{gi}$'s and $\lambda_{ik}$'s as dual variables,
then the augmented Lagrangian is constructed as 
\begin{align}
\notag
\min_{\{\!e_{gi}\!,\alpha_{gi}\!\},\!\{\!z_{ik}\!,u_{ik}\!,\lambda_{ik}\!\}}&
\frac{1}{2}
\sum_{g=1}^{G}\left( 
\NM{w_{gi}\odot e_{gi} }{2}^2
\!\!+\!\! \alpha_{gi}^\top
\left(e_{gi}\!\!-\!\!{x}_i\!\!+\!\!\sumDZ\!\!+\!\!\mu_g\right)
\!\!+\!\!\frac{\rho}{2}\NML{e_{gi}\!\!-\!\!{x}_i\!\!+\!\!\sumDZ\!\!+\!\!\mu_g}{2}^2\right) \\\notag
&\!\!+\!\!
\sum_{k=1}^{K}\!\!
\left( 
\beta\NM{u_{ik}}{1}
\!\!+\!\! 
\lambda_{ik}^\top
(z_{ik} \!\!-\!\! u_{ik})
\!\!+\!\! \frac{\rho}{2}
\NM{z_{ik} \!\!-\!\! u_{ik}}{2}^2 \!\right) \!\!.
\end{align}
At $\tau$th iteration, ADMM alternately updates 
$z^\tau_{ik}$'s, $e_{gi}^\tau$'s, $u^\tau_{ik}$'s, $\alpha^{\tau}_{gi}$'s and $\lambda^{\tau}_{ik}$'s until convergence.


$z^{\tau}_{ik}$'s are updated as
\begin{align}
\notag
\lefteqn{
	\arg\min_{\{z_{ik}\}}
	\sum_{g=1}^G\left( 
	(\alpha_{gi}^{\tau-1})^\top
	\left(e_{gi}^{\tau-1}\!-\!{x}_i\!+\!\sum_{k=1}^{K}d_k*z_{ik}+\mu_g\right)
\!+\!\frac{\rho}{2}\NML{e_{gi}^{\tau-1}\!-\!{x}_i\!+\!\sum_{k=1}^{K}d_k*z_{ik}+\mu_g}{2}^2\right) }
\\\notag
&
+ \sum_{k=1}^{K}\left( 
(\lambda^{\tau-1}_{ik})^\top
(z_{ik} \!-\! u^{\tau-1}_{ik})
\!+\! \frac{\rho}{2}
\NM{z_{ik} \!-\! u^{\tau-1}_{ik}}{2}^2\right) \\\label{eq:wcsc_admm_z}
&=\arg 
\min_{\{z_{ik}\}}
\sum_{g=1}^G
\NML{
	{x}_i-\sum_{k=1}^{K}d_k*z_{ik} -\mu_g-e^{\tau-1}_{gi}-\frac{\alpha^{\tau-1}_{gi}}{\rho} }{2}^2
+\sum_{k=1}^{K} 
\NML{z_{ik}-u^{\tau-1}_{ik}+\frac{\lambda^{\tau-1}_{ik}}{\rho}}{2}^2.
\end{align}
Let $\eta^{\tau-1}_{gi} = x_i-\mu_g-e^{\tau-1}_{gi}-\frac{\alpha^{\tau-1}_{gi}}{\rho}$, and transform all variables to frequency domain as
$\hat{d}_k =\FFT(C^\top d_k)$, 
$\hat{z}_{ik} = \FFT(z_{ik})$, $\hat{u}_{ik}^{\tau-1} = \FFT(u_{ik}^{\tau-1})$, $\hat{\lambda}_{ik}^{\tau-1} = \FFT(\lambda_{ik}^{\tau-1})$,  
$\hat{\eta}_{gi}^{\tau-1} = \FFT(\eta_{gi}^{\tau-1})$ and
$\hat{\alpha}_{gi}^{\tau-1} = \FFT(\alpha_{gi}^{\tau-1})$.
Then \eqref{eq:wcsc_admm_z} can be written as: 
\begin{align*}
\min_{\{\hat{z}_{ik}\}}
\sum_{g=1}^G
\NML{
	\hat{\eta}^{\tau-1}_{gi}
	\!-\!
	\sumDZfre}{2}^2
\!+\!
\sum_{k=1}^{K}
\NML{
	\hat{z}_{ik}\!-\!\hat{u}^{\tau-1}_{ik}\!+\!\frac{\hat{\lambda}^{\tau-1}_{ik}}{\rho} }{2}^2.
\end{align*}
Using the reordering trick, we define
$\hat{D}=[\hat{d}_1,\dots,\hat{d}_K]$, 
$\hat{Z}_{i}=[\hat{z}_{i1},\dots,\hat{z}_{iK}]$, 
$\hat{U}^{\tau-1}_{i}=[\hat{u}^{\tau-1}_{i1},\dots,\hat{u}^{\tau-1}_{iK}]$,
and 
$\hat{\Lambda}^{\tau-1}_{i}=[\hat{\Lambda}^{\tau-1}_{i1},\dots,\hat{\Lambda}^{\tau-1}_{iK}]$.
Then $p$th row $\hat{Z}^{\tau}_{i}(p,:)$ is updated as
\begin{align*}
\hat{Z}^{\tau}_{i}(p,:)=
\left(\zeta(p)\hat{D}^{\star}(p,:) \!+\! \hat{U}^{\tau-1}_{i}(p,:) \!-\!\frac{\hat{\Lambda}^{\tau-1}_{i}(p,:)}{\rho}\right)
\cdot
(G\hat{D}^{\top}(:,p)\hat{D}^\star(p,:)+I)^{-1},
\end{align*}
where $\zeta=\sum_{g=1}^G\hat{\eta}_{gi}^{\tau-1}(p)$.
$z^\tau_{ik}$ is recovered as $\iFFT(\hat{z}^\tau_{ik})$. 


Each $e^{\tau}_{gi}$ is independently updated as
\begin{align}\notag
	\arg&\min_{e_{gi}}
	\frac{1}{2}
	\NM{w_{gi}\odot e_{gi} }{2}^2	+
(\alpha_{gi}^{\tau-1})^\top
\left(e_{gi}-{x}_i+\sum_{k=1}^{K}d_k*z^\tau_{ik}+\mu_g\right)
+\!\frac{\rho}{2}\NML{e_{gi}\!-\!{x}_i\!+\!\sum_{k=1}^{K}d_k*z^\tau_{ik}+\mu_g}{2}^2
\\\notag
&=	\arg\min_{e_{gi}}
\NM{w_{gi}\odot e_{gi} }{2}^2+\rho 
\NML{
	{x}_i\!-\!\sum_{k=1}^{K}d_k*z^\tau_{ik}-\mu_g\!-\!e_{gi}\!-\!\frac{\alpha^{\tau-1}_{gi}}{\rho} }{2}^2.
\end{align}
Then similar to \eqref{eq:wcsc_admm_e}, $e^{\tau}_{gi}(p)$ is updated as 
\begin{equation*}
e^{\tau}_{gi}(p)= \frac{\rho \nu(p)}{w^2_{gi}(p)+\rho},
\end{equation*}
where $\nu = {x}_i-\sum_{k=1}^{K}d_k*z^\tau_{ik} -\mu_g- \frac{\alpha^{\tau-1}_{gi}}{\rho}$.

Each $u^{\tau}_{ik}$ is updated as
\begin{align}\notag
\!\!\lefteqn{\arg\min_{u_{ik}}
	(\lambda^{\tau-1}_{ik})^\top
	(z_{ik}^{\tau} \!-\! u_{ik})
	\!+\! \frac{\rho}{2}
	\NM{z_{ik}^\tau \!-\! u_{ik}}{2}^2
\!+\!\beta\|u_{ik}\|_1}\\\notag
&=\arg 
\min_{u_{ik}}
\beta\|u_{ik}\|_1\!+\!
\frac{\rho}{2}
\NML{z_{ik}^\tau\!-\!u_{ik}\!+\!\frac{\lambda^{\tau-1}_{ik}}{\rho}}{2}^2.
\end{align}
with closed-form solution:
\begin{equation*}\label{eq:cscl1_code_u_sol}
u^{\tau}_{ik}(p) =\text{sign}(\psi(p))\odot\max(|\psi(p)| - \frac{\beta}{\rho}, 0)
\end{equation*}
where $\psi = z^{\tau}_{ik}+\frac{\lambda^{\tau-1}_{ik}}{\rho}$.

Finally,
$\alpha^{\tau}_{gi}$ and
$\lambda^{\tau}_{ik}$ 
are updated as:
\begin{align*}
\alpha^{\tau}_{gi} &=\alpha^{\tau-1}_{gi}+\rho \left(e^{\tau}_{gi} - {x}_{i} + \sum_{k=1}^{K}d_k*z^{\tau}_{ik}+\mu_g \right),\\
\lambda^{\tau}_{ik}&= \lambda^{\tau-1}_{ik}+ \rho (z^{\tau}_{ik}-u^{\tau}_{ik}).
\end{align*}

\end{document}